\documentclass{article}

\PassOptionsToPackage{numbers, compress}{natbib}

\usepackage[preprint]{nips_2018}

\usepackage[utf8]{inputenc}
\usepackage[numbers]{natbib} 
\usepackage[T1]{fontenc}    
\usepackage{hyperref}       
\usepackage{url}            
\usepackage{booktabs}       
\usepackage{amsfonts}       
\usepackage{nicefrac}       
\usepackage{microtype}      
\usepackage{amsfonts}       
\usepackage{nicefrac}       
\usepackage{microtype}      
\usepackage{times}
\usepackage{adjustbox}
\usepackage{tabularx}
\usepackage{amsmath}
\usepackage{amsthm}
\usepackage{blindtext}
\usepackage{enumitem}
\usepackage{multirow}
\usepackage{graphicx}
\usepackage{algorithm}
\usepackage{algpseudocode}
\usepackage{amsfonts}
\usepackage{amsthm}
\theoremstyle{definition}

\newtheorem{definition}{Definition}
\newtheorem{theorem}{Theorem}
\newtheorem{corollary}{Corollary}
\newtheorem{lemma}{Lemma}
\newtheorem*{theorem*}{Theorem}
\newtheorem*{lemma*}{Lemma}

\usepackage{color}
\usepackage{wrapfig}
\newcommand{\RN}[1]{%
	\textup{\lowercase\expandafter{\it \romannumeral#1}}%
}
\title{Empirical Risk Minimization in Non-interactive Local Differential Privacy: Efficiency and High Dimensional Case}

\author{
	Di Wang\qquad Marco Gaboardi \qquad Jinhui Xu \\
	Department of Computer Science and Engineering\\
	State University of New York at Buffalo\\
	Buffalo, NY, 14260
}

\begin{document}

\maketitle

\begin{abstract}
In this paper, we study the Empirical Risk Minimization problem in the non-interactive local model of differential privacy. In the case of constant or low dimensionality ($p\ll n$), we first show that if the ERM loss function is $(\infty, T)$-smooth,  then we can avoid a dependence of the  sample complexity, to achieve error $\alpha$, on the exponential of the dimensionality $p$ with base $1/\alpha$  ({\em i.e.,} $\alpha^{-p}$),
 which answers a question in \cite{smith2017interaction}.  Our approach is based on polynomial approximation. Then, we propose player-efficient algorithms with $1$-bit communication complexity and $O(1)$ computation cost for each player. The error bound is asymptotically the same as the original one. Also with additional assumptions we show a server efficient algorithm. Next we consider the high dimensional case ($n\ll p$), we show that if the loss function is Generalized Linear function and convex, then we could get an error bound which is dependent on the Gaussian width of the underlying constrained set instead of $p$, which is lower than that in \cite{smith2017interaction}.
\end{abstract}

\section{Introduction}

Differential privacy \cite{dwork2006calibrating} has emerged as a rigorous notion for privacy which allows accurate data analysis with a guaranteed bound on the increase in harm for each individual to contribute her data. 
Methods to guarantee differential privacy have been widely studied, and recently adopted in  industry~\cite{208167,erlingsson2014rappor}.\par

Two main user models have emerged for differential privacy: the central model and the local one. In the local model, each individual manages his/her proper data and discloses them to a server through some differentially private mechanisms. The server collects the (now private) data of each individual and combines them into a resulting  data analysis. A classical use case for this model is the one aiming at collecting statistics from user devices like in the case of Google's Chrome browser~\cite{erlingsson2014rappor}, and  Apple's iOS-10 \cite{208167,DBLP:journals/corr/abs-1709-02753}.

In the local model, there are two basic kinds of protocols: interactive and non-interactive.  \citet{bassily2015local} have recently investigated the power of non-interactive differentially private protocols. These protocols are more natural for the classical use cases of the local model: both the projects from Google and Apple use the non-interactive model. Moreover, implementing efficient interactive protocols in such applications is more difficult due to the latency of the network and communication cost. Despite being used in industry, the local model has been much less studied than the central one. Part of the reason for this is that there are intrinsic limitations in what one can do in the local model. As a consequence, many basic questions, that are well studied in the central model, have not been completely understood in the local model, yet. 

In this paper, we study differentially private  Empirical Risk Minimization in the  non-interactive local model.  Before showing our contributions and  discussing comparisons with previous works, we firstly discuss our motivations.

\paragraph{Problem setting \cite{smith2017interaction}\cite{kasiviswanathan2011can}} Given a  convex, closed and bounded constraint set $\mathcal{C}\subseteq \mathbb{R}^{p}$, a data universe $\mathcal{D}$, and a loss function $\ell:\mathcal{C}\times \mathcal{D}\mapsto \mathbb{R}$. A dataset $D=\{x_1,x_2\cdots,x_n\}\in \mathcal{D}^n$ defines an \emph{empirical risk} function: $\hat{L}(\theta;D)=\frac{1}{n}\sum_{i=1}^{n}\ell(\theta,x_i)$. When the inputs are drawn i.i.d from an unknown underlying distribution $\mathcal{P}$ on $\mathcal{D}$, we can also define the \emph{population risk} function: $L_\mathcal{P}(\theta)=\mathbb{E}_{D\sim \mathcal{P}^n}[\ell(\theta;D)]$. Now we have the following two kinds of excess risk, one is empirical risk, {\em i.e.}
$\text{Err}_{D}(\theta_{\text{priv}})=\hat{L}(\theta_{\text{priv}};D)-\min_{\theta\in \mathcal{C}}\hat{L}(\theta;D)$, the other one is population risk, {\em i.e.}
$\text{Err}_{\mathcal{P}}(\theta_{\text{priv}})=L_{\mathcal{P}}(\theta_{\text{priv}})-\min_{\theta\in \mathcal{C}} L_{\mathcal{P}}(\theta).$\par 
The problem that we study in this paper is for finding $\theta_{\text{priv}}\in \mathcal{C}$ under non-interactive local differential privacy (see Definition \ref{def:1}) which makes the empirical and population excess risk as low as possible. Alternatively, when dimensionality $p$ is constant or low, we can express this problem in terms of \emph{sample complexity} as finding  as small of $n$ as possible for achieving $\text{Err}_{D}\leq \alpha$ and $\text{Err}_{\mathcal{P}}\leq \alpha$, where $\alpha$ is the user specified error tolerance (or simply called error). 
\paragraph{Motivation}
 \citet{smith2017interaction} prove the following result concerning the problem for general convex 1-Lipschitz loss functions over a bounded constraint set.
\begin{theorem}\label{theorem:1}
	Under the assumptions above, there is a non-interactive $\epsilon$-LDP algorithm such that for all distribution $\mathcal{P}$ on $\mathcal{D}$, with probability $1-\beta$, we have
	\begin{equation}\label{equation:1}
	 \text{Err}_{\mathcal{P}}(\theta_{\text{priv}})\leq \tilde{O}\big((
	\frac{\sqrt{p}\log^2(1/\beta)}{\epsilon^2 n}
	)^{\frac{1}{p+1}}\big).	\end{equation}
 A similar result holds for $\text{Err}_{D}$, with at least $\Omega(n^{\frac{1}{p+1}})$ for both computation and communication complexity for each user. Alternatively, to achieve error $\alpha$, the sample complexity must satisfies $n=\tilde{\Omega}(\sqrt{p}c^p\epsilon^{-2} \alpha^{-(p+1)})$, where $c$ is some constant (approximately 2). More importantly, they also show that generally, the dependence of the sample size over the dimensionality $p$, in the terms $\alpha^{-(p+1)}$ and $c^p$, is unavoidable.
\end{theorem}
This situation is somehow undesirable: when the dimensionality is high and the target error is low, the dependency on  $\alpha^{-(p+1)}$ could make the sample size quite large. However, several results have already shown that for some specific loss functions, the exponential dependency on the dimensionality can be avoided. For example,
 \citet{smith2017interaction} show that, in the case of linear regression, there is a non-interactive $(\epsilon,\delta)$-LDP algorithm\footnote{Although, these two results are formulated for non-interactive $(\epsilon,\delta)$-LDP, in the rest of the paper we will focus on non-interactive $\epsilon$-LDP algorithms.} whose sample complexity for achieving error $\alpha$ for the empirical risk is $n=\Omega(p\log(1/\delta)\epsilon^{-2}\alpha^{-2})$. Similarly,
 \citet{DBLP:conf/icml/0007MW17} showed that for logistic regression, if the sample complexity satisfies $n> O\big((
	\frac{8r}{\alpha})^{4r\log\log(8r/\alpha)}(\frac{4r}{\epsilon})^{2cr\log(8r/\alpha)+2}(\frac{1}{\alpha^2 \epsilon^2})\big),$ where $c$ and $r$ are independent on $p$, then there is an non-interactive $(\epsilon,\delta)$-LDP such that $\text{Err}_{\mathcal{P}}(\theta_{\text{priv}})\leq \alpha$.
\par 

In this paper we will firstly study the following natural questions: $\RN{1})$ Can we get an algorithm which has lower sample complexity than in Theorem \ref{theorem:1}? $\RN{2})$
From the discussion above, we have a gap between the general case and the case of specific loss functions. 
Can we give natural conditions on the loss function that guarantee non-interactive $\epsilon$-LDP with sample complexity that is not exponential in the dimensionality $p$? $\RN{3})$ As we can see from above the computation and communication cost is relatively high when $n$ is large, can we reduce them to constant? $\RN{3})$
Next, we consider the problem in high dimensional case. \citet{smith2017interaction} assumes that the dimensionality is low or constant compared with $n$, however, in machine learning it is common when in the high dimensional case, that is $n\ll p$. We can see the above bound is meaningless under this case. So the question is how can we get a lower upper bound?

\paragraph{Our Contributions} 

\begin{enumerate}
\item For low (constant) dimensional case, we first  show that there is a non-interactive $\epsilon$-LDP algorithm, if the loss function is $(8, T)$-smooth (Definition \ref{def:5}),  then when $n=\tilde{\Omega}\big( (c_0p^{\frac{1}{4}})^p\alpha^{-(2+\frac{p}{2})}\epsilon^{-2}\big)$, where $c_0$ is a universal constant, then the empirical excess risk will satisfies $\text{Err}_{D}\leq \alpha$. 
If the loss function is $(\infty, T)$-smooth, then when  $n\geq \tilde{\Omega}(4^{p(p+1)}D^2_p p\epsilon^{-2}\alpha^{-4})$, we have empirical excess risk $\text{Err}_{D}\leq \alpha$, where $D_p$ depends only on $p$. Interestingly, to obtain this result we do not need the loss function to be convex. However, 
if the loss function is convex and 1-Lipschitz,  results of population excess risk can also be achieved.  Note that when $\alpha\leq O(\frac{1}{p})$ the complexity of our first result is lower than it in \cite{smith2017interaction}, also in the latter result the dependence on $\alpha$ is constant, respectively, rather than $\alpha^{-(p+1)}$. Our method is based on Berstein polynomial approximation.

\item Next, we address the efficiency issue, 
which has not been well 
studied before \cite{smith2017interaction}. Following an approach similar to~\cite{bassily2015local},  we propose an algorithm for our loss functions which has only $1$-bit communication cost and $O(1)$ computation cost for each client, and which achieves asymptotically the same error bound as the original one. Additionally, we show also a novel analysis for the server. This shows that if the loss function is convex and Lipschitz and the convex set satisfies some natural conditions, then we have an algorithm which achieves the error bound of $O(p\alpha)$ when $n$ is the same as in the previous part, moreover, the  running time is polynomial in $\frac{1}{\alpha}$ if the loss function is $(\infty, T)$-smooth, which is exponential in $d$ in \cite{smith2017interaction}. 
\item Later, we show the generality of our technique by applying polynomial approximation to other problems. We give a non-interactive LDP algorithm for answering the class of k-way marginals queries and the class of smooth queries, by using different type of polynomials approximations (details are in Appendix). 
\item 
For high dimensional case, we show that if the loss function is a convex general linear function, then we have an $\epsilon$-LDP algorithm whose risk bound is only dependent on $n$ and the Gaussian Width of $\mathcal{C}$, this is much smaller than it in \cite{smith2017interaction}. When $\mathcal{C}$ is $\ell_1$ norm ball or distribution simplex, we show it will only dependent on $n, \log p$ instead of $p$. 
\end{enumerate}
\begin{table*}[h]
    \large
	\begin{center}
		\resizebox{1\linewidth}{!}{%
			\begin{tabular}[t]{
			|p{1.5cm}|p{4cm}|p{2.5cm}|p{2cm}|p{2cm}|p{3cm}|}
				\hline
				Method & Sample Complexity (omit $\text{Poly}(p)$ terms)& Communication Cost(each user) &  Computation Cost(each user) & Running time for the server & Assumption\\[0.5ex]
				\hline
				Claim 4 in \cite{smith2017interaction} & $\tilde{\Omega}(4^p\alpha^{-(p+2)}\epsilon^{-2})$ & $1$ & $O(1)$ & $O\big((\frac{1}{\alpha})^p\big)$ & Lipschitz\\ [1ex]
				\hline
				Theorem 10 in \cite{smith2017interaction} & $\tilde{\Omega}(2^p\alpha^{-(p+1)}\epsilon^{-2})$ & $\Omega(n^{\frac{1}{p+1}})$ & $\Omega(n^{\frac{1}{p+1}})$ & Not Mentioned & Lipschitz and Convex \\ [1ex]
				\hline
				\textbf{This Paper} &  $\tilde{\Omega}\big( (c_0 p^{\frac{1}{4}})^p\alpha^{-(2+\frac{p}{2})}\epsilon^{-2}\big)$ & $1$ & $O(1)$ & $O( (\frac{1}{\alpha})^{\frac{p}{2}})$ & $(8, T)$-smooth \\ [1ex]
				\hline 
				\textbf{This Paper} &  $\tilde{\Omega}(4^{p(p+1)}D^2_p\epsilon^{-2}\alpha^{-4})$ & $1$ & $O(1)$ & $O\big(\text{Poly}(\frac{1}{\alpha})\big)$ &  $(\infty, T)$-smooth \\ [1ex]
				\hline
		\end{tabular}}
	\end{center}
	\caption{Comparisons with previous works on the empirical risk under low dimensional case. We can see that when the error $\alpha \leq O(\frac{1}{p})$ then the sample complexity of $(8, T)$-smooth loss function case is less than previous works. When the error $\alpha\leq O(\frac{1}{16^p})$, then the sample complexity for $(\infty, T)$-smooth loss function case is less than previous works.}
	\label{Table:1}
\end{table*}

We list some of our results in Table \ref{Table:1}. Due to the space limit, all the proofs and some details of algorithms can be found in the Appendix part. Also, in order for convenience, we have to note that many of the upper bounds are quite loose.

\section{Related Works}

ERM in the local model of differential privacy has been studied in \cite{kasiviswanathan2011can,beimel2008distributed,duchi2017minimax,duchi2013local,DBLP:conf/icml/0007MW17,smith2017interaction}. \citet{kasiviswanathan2011can} showed a general equivalence between learning in the local model and learning in the  statistical query model. \citet{duchi2017minimax,duchi2013local} gave the lower bound $O(\frac{\sqrt{d}}{\epsilon\sqrt{n}})$ and optimal algorithms for general convex optimization; however, their optimal procedure needs many rounds of interactions. The works that are most related to ours are \cite{DBLP:conf/icml/0007MW17,smith2017interaction}. \citet{DBLP:conf/icml/0007MW17} considered some specific loss functions in high dimensions, such as sparse linear regression and kernel ridge regression, 
Note that although it also studied 
a class of loss functions ({\em i.e.,} Smooth Generalized Linear Loss functions) and used the polynomial approximation approach, the functions investigated in our paper are more general, 
which include linear regression and logistic regression, and the approximation techniques are quite different.
\citet{smith2017interaction} studied general convex loss functions for population excess risk and showed that the dependence on the exponential of the dimensionality is unavoidable. In this paper, we show that such a dependence in the term of $\alpha$ is actually avoidable for a class of loss functions, and this even holds for non-convex loss functions, which is a big difference from all existing works, also we consider high dimensional case. In addition, our algorithms are simpler and more efficient.
 The polynomial approximation approach has been used under central model in \cite{alda2017bernstein,wang2016differentially,thaler2012faster,DBLP:conf/icml/0007MW17} and the dimension reduction has been used in local model in \cite{bassily2015local,DBLP:conf/icml/0007MW17}. 
\section{Preliminaries}
\paragraph{Differential privacy in the local model.} In LDP, we have a data universe $\mathcal{D}$,  $n$ players with each holding  a private data record $x_i\in \mathcal{D}$, and a server that is in charge of coordinating the protocol. An LDP protocol proceeds in $T$ rounds. In each round, the server sends a message, which we sometime call a query, to a subset of the players, requesting them to run a particular algorithm. Based on the queries, each player $i$ in the subset selects an algorithm $Q_i$, run it on her data, and sends the output back to the server.
\begin{definition}\cite{kasiviswanathan2011can,smith2017interaction}\label{def:1}
An algorithm $Q$ is $\epsilon$-locally differentially private (LDP) if for all pairs $x,x'\in \mathcal{D}$, and for all events $E$ in the output space of $Q$, we have $\text{Pr}[Q(x)\in E]\leq e^{\epsilon}\text{Pr}[Q(x')\in E].$
A multi-player protocol is $\epsilon$-LDP if for all possible inputs and runs of the protocol, the transcript of player i's interaction with the server is $\epsilon$-LDP. If $T=1$, we say that the protocol is $\epsilon$ non-interactive LDP.
\end{definition}
  \begin{wrapfigure}{R}{0.5\linewidth}\vspace{-1cm}
	\begin{minipage}{\linewidth}
       \begin{algorithm}[H]
	   \caption{1-dim LDP-AVG}
	   \label{alg:1}
	   \begin{algorithmic}[1]
	   \State {\bfseries Input:} Player $i\in [n]$ holding data $v_i\in [0,b]$, privacy parameter $\epsilon$.
	   \For{Each Player $i$}
	   \State
		Send $z_i=v_i+\text{Lap}(\frac{b}{\epsilon})$
	    \EndFor
	    \For{The Server}
		\State Output $a=\frac{1}{n}\sum_{i=1}^{n}z_i$.
		\EndFor
	\end{algorithmic}
\end{algorithm}
		\vspace{-0.7cm}
	\end{minipage}
\end{wrapfigure}

Since we only consider non-interactive LDP through the paper, we will use LDP as non-interactive LDP below. 
As an example that will be useful in the sequel, the next lemma shows an $\epsilon$-LDP algorithm for computing 1-dimensional average.
\begin{lemma}\label{lemma:1}
	Algorithm \ref{alg:1} is $\epsilon$-LDP. Moreover, if player $i\in [n]$ holds value $v_i\in [0,b]$ and $n>\log \frac{2}{\beta}$  with $0<\beta<1$, then, with probability at least $1-\beta$, the output $a\in \mathbb{R}$ satisfies:
	$|a-\frac{1}{n}\sum_{i=1}^{n}v_i|\leq \frac{2b\sqrt{\log \frac{2}{\beta}}}{\sqrt{n}\epsilon}$.
\end{lemma}
\paragraph{Bernstein polynomials and approximation.}

We give here some basic definitions that  will be used in the sequel; more details can be found in \cite{alda2017bernstein,lorentz1986bernstein,micchelli1973saturation}.

\begin{definition}\label{def:2}
	Let $k$ be a positive integer. The Bernstein basis polynomials of degree $k$ are defined as $b_{v,k}(x)=\binom{k}{v}x^{v}(1-x)^{k-v}$ for $v=0,\cdots,k$.
\end{definition}
\begin{definition}\label{def:3}
	Let $f:[0,1]\mapsto \mathbb{R}$ and $k$ be a positive integer. Then, the Bernstein polynomial of $f$ of degree $k$ is defined as $B_k(f;x)=\sum_{v=0}^{k}f(v/k)b_{v,k}(x)$.  
We denote by $B_k$ the Bernstein operator $B_k(f)(x)=B_k(f,x)$.

\end{definition}
\begin{definition}\cite{micchelli1973saturation} \label{def:4}
	Let $h$ be a positive integer. The iterate Bernstein operator of order $h$ is defined as the sequence of linear operators $B_k^{(h)}=I-(I-B_k)^h=\sum_{i=1}^{h}\binom{h}{i}(-1)^{i-1}B_k^i$, where $I=B_k^0$ denotes the identity operator and $B_k^i$ is defined as $B_k^i=B_k\circ  B_k^{k-1}$. The iterated Bernstein polynomial of order $h$ can be computed as 
	$B_k^{(h)}(f;x)=\sum_{v=0}^{k}f(\frac{v}{k})b_{v,k}^{(h)}(x),$
	where $b^{(h)}_{v,k}(x)=\sum_{i=1}^{h}\binom{h}{i}(-1)^{i-1}B^{i-1}_k(b_{v,k};x)$.
\end{definition}
Iterate Bernstein operator can well approximate multivariate $(h,T)$-smooth functions.
\begin{definition}\cite{micchelli1973saturation} \label{def:5}
	Let $h$ be a positive integer and $T>0$ be a constant. A function $f:[0,1]^{p}\mapsto \mathbb{R}$ is $(h,T)$-smooth if it is in class $\mathcal{C}^{h}([0,1]^{p})$ and its partial derivatives up to order $h$ are all bounded by $T$. We say it is $(\infty,T)$-smooth, if for every $h\in \mathbb{N}$ it is $(h,T)$-smooth.
\end{definition}
\begin{definition}\label{def:6}
	Assume $f:[0,1]^p \mapsto \mathbb{R}$ and let $k_1,\cdots,k_p,h$ be positive integers. The multivariate iterated Bernstein polynomial of order $h$ at $y=(y_1,\ldots,y_p)$ is defined as:
	\begin{equation}\label{equation:2}
	B^{(h)}_{k_1,\ldots, k_p}(f;y)=\sum_{j=1}^{p}\sum_{v_j=0}^{k_j}f(\frac{v_1}{k_1},\ldots,\frac{v_p}{k_p})\prod_{i=1}^p b^{(h)}_{v_i,k_i}(y_i).
	\end{equation}

	We denote 	$B^{(h)}_k=B^{(h)}_{k_1,\ldots, k_p}(f;y)$ if $k=k_1=\cdots=k_p$.
\end{definition}
\begin{theorem}\cite{alda2017bernstein}\label{theorem:2}
		If $f:[0,1]^p\mapsto \mathbb{R}$ is a $(2h,T)$-smooth function, then for all positive integers $k$ and $y\in[0,1]^p$, we have $|f(y)-B_k^{(h)}(f;y)|\leq O(pTD_h k^{-h})$. Where $D_h$ is a universal constant only related to $h$.
\end{theorem}

\paragraph{Our settings}

We conclude this section by making explicitly the settings that we will consider throughout the paper. We assume that there is a constraint set $\mathcal{C}\subseteq [0,1]^p$ and for every $x\in \mathcal{D}$ and $\theta\in \mathcal{C}$,  $\ell(\cdot,x)$ is well defined on $[0,1]^p$ and $\ell(\theta,x)\in [0,1]$.
These closed intervals can be extended to arbitrarily bounded closed intervals.
Our settings are similar to the `Typical Settings' in \cite{smith2017interaction}, where  $\mathcal{C}\subseteq [0,1]^p$ appears in their Theorem 10, and  $\ell(\theta,x)\in [0,1]$ from their 1-Lipschitz requirement and $\|\mathcal{C}\|_2\leq 1$.
\section{Low Dimensional Case}
Definition \ref{def:6} and Theorem \ref{theorem:2} tell us that if we know the value of the empirical risk function, {\em i.e.} the average of the sum of loss functions, on each of the grid points $(\frac{v_1}{k},\frac{v_2}{k}\cdots\frac{v_p}{k})$, where $(v_1,\cdots,v_p)\in \mathcal{T}=\{0,1,\cdots,k\}^p$ for some large $k$, then we can approximate it well. Our main observation is that this can be done in the local model by estimating the average of the sum of loss functions on each of the grid points using Algorithm~\ref{alg:1}. This is the idea of Algorithm \ref{alg:2}.
\begin{algorithm}
	\caption{Local Bernstein Mechanism}
	\label{alg:2}
	\begin{algorithmic}[1]
		\State {\bfseries Input:} Player $i\in [n]$ holding data $x_i\in \mathcal{D}$, public loss function $\ell:[0,1]^p \times \mathcal{D}\mapsto [0,1]$, privacy parameter $\epsilon>0$, and parameter $k$.
		\State Construct the grid $\mathcal{T}=\{\frac{v_1}{k},\ldots,\frac{v_p}{k}\}_{\{v_1,\ldots,v_p\}}$, where $\{v_1,\ldots,v_p\}\in\{0,1,\cdots,k\}^p$.
		\For {Each grid point $v=(\frac{v_1}{k},\ldots,\frac{v_p}{k})\in \mathcal{T}$}
		\For{Each Player $i\in [n]$}
		\State
		Calculate $\ell(v;x_i)$.
		\EndFor
		\State Run Algorithm \ref{alg:1} with $\epsilon=\frac{\epsilon}{(k+1)^p}$ and $b=1$ and denote the output as $\tilde{L}(v;D)$.
		\EndFor
		\For{The Server}
		\State Construct Bernstein polynomial, as in (\ref{equation:2}), for the perturbed empirical loss $\tilde{L}(v;D)$. Denote $\tilde{L}(\cdot,D)$ the corresponding function. 
		\State Compute $\theta_{\text{priv}}=\arg\min_{\theta\in \mathcal{C}}\tilde{L}(\theta;D)$.
		\EndFor
	\end{algorithmic}
\end{algorithm}
\begin{theorem}\label{theorem:3}
	For $\epsilon>0, 0<\beta<1$, Algorithm \ref{alg:2} is $\epsilon$-LDP. Assume that the loss function $\ell(\cdot,x)$ is $(2h,T)$-smooth for all $x\in \mathcal{D}$ for some positive integer $h$ and constant $T$. If $n,\epsilon$ and $\beta$ satisfy $n=\Omega\Big (\frac{\log \frac{1}{\beta}4^{p(h+1)}}{\epsilon^2 D_{h}^2}\Big )$, then setting $k=O\Big((\frac{D_h\sqrt{pn}\epsilon}{2^{(h+1)p}\sqrt{\log \frac{1}{\beta}}})^{\frac{1}{h+p}}\Big)$ we have with probability at least $1-\beta$:
	\begin{equation}\label{equation:3}
	\text{Err}_{D}(\theta_{\text{priv}})\leq
	\tilde{O}\Big (\frac{\log^{\frac{h}{2(h+p)}} (\frac{1}{\beta}) D_h^{\frac{p}{p+h}}p^{\frac{p}{2(h+p)}}2^{(h+1)p\frac{h}{h+p}}}{n^{\frac{h}{2(h+p)}}\epsilon^{\frac{h}{h+p}}}\Big ),
	\end{equation}
	where $\tilde{O}$ hides the $\log$ and $T$ terms.
\end{theorem}
From (\ref{equation:3}) we can see that in order to achieve error $\alpha$, the sample complexity needs to be $n=\tilde{\Omega}(\log \frac{1}{\beta}D_h^{\frac{2p}{h}}p^{\frac{p}{h}}4^{(h+1)p}\epsilon^{-2}\alpha^{-(2+\frac{2p}{h})})$. As particular cases, we have the followings.
\begin{corollary}\label{col1}
If the loss function $\ell(\cdot,x)$ is $(8,T)$-smooth for all $x\in \mathcal{D}$ for some constant $T$, nd if $n,\epsilon, \beta, k$ satisfy the condition in Theorem~\ref{theorem:3} with $h=4$, then with probability at least $1-\beta$, the sample complexity to achieve $\alpha$ error is $n=\tilde{O}\big(\alpha^{-(2+\frac{p}{2})}\epsilon^{-2}(4^5\sqrt{D_4}p^{\frac{1}{4}})^p\big)$, where for general convex loss function $n=\tilde{\Omega}\big(\alpha^{-(p+1)}\epsilon^{-2}2^p\big)$ in Theorem \ref{theorem:1}. We can easily get that when the error satisfies $\alpha \leq O(\frac{1}{p})$, the sample complexity is lower than it in \cite{smith2017interaction}, we note that this case always appears in real applications (see below).
\end{corollary}

\begin{corollary}\label{col2}
	If the loss function $\ell(\cdot,x)$ is $(\infty,T)$-smooth for all $x\in \mathcal{D}$ for some constant $T$, and if $n,\epsilon, \beta, k$ satisfy the condition in Theorem~\ref{theorem:3} with $h=p$, then with probability at least $1-\beta$, the output $\theta_{\text{priv}}$ of Algorithm \ref{alg:2} satisfies:
$\text{Err}_{D}(\theta_{\text{priv}})\leq
	\tilde{O}\Big (\frac{\log \frac{1}{\beta}^{\frac{1}{4}}D_p^{\frac{1}{2}}p^{\frac{1}{4}}\sqrt{2}^{(p+1)p}}{n^{\frac{1}{4}}\epsilon^{\frac{1}{2}}}\Big)$,
	where $\tilde{O}$ hides the $\log$ and $T$ terms. So, to achieve error $\alpha$, with probability at least $1-\beta$, we have sample complexity: 
	\begin{equation}\label{equation:4}
	n=\tilde{\Omega}\Big (\max\{4^{p(p+1)}\log(\frac{1}{\beta})D_p^2 p \epsilon^{-2}\alpha^{-4}, \frac{\log \frac{1}{\beta}4^{p(p+1)}}{\epsilon^2 D_{p}^2}  \}\Big ),
	\end{equation}
\end{corollary}
It is worth noticing that from (\ref{equation:3}), when the term $\frac{h}{p}$ grows, the term $\alpha$ decreases. Thus, for loss functions that are $(\infty,T)$-smooth, we can get a smaller dependency than the term $\alpha^{-4}$ in (\ref{equation:4}). For example, if we take $h=2p$, then the sample complexity is $n=\Omega(\max\{c_2^{p^2}\log \frac{1}{\beta}D_{2p} \sqrt{p} \epsilon^{-2}\alpha^{-3}, \frac{\log \frac{1}{\beta}c^{p^2}}{\epsilon^2 D_{2p}^2}  \})$ for some constants $c, c_2$. When $h\rightarrow \infty$, the dependency on the error becomes $\alpha^{-2}$, which is the optimal bound, even for convex functions. 

Our analysis of the empirical excess risk does not use the convexity assumption. While this gives a bound which is not optimal, even for $p=1$, it also says that our result holds for non-convex loss functions and constrained domain set, as long as they are smooth enough.\par 
By (\ref{equation:4}), we can see that our sample complexity is lower than it in \cite{smith2017interaction} when $\alpha\leq O(\frac{1}{16^p})$ (we assume $D_p$ is a constant here since $p$ is a constant). We have to note that this is always the case when studying ERM when the dimension $d$ is small, in order to get best performance usually, we wan to achieve the error of $\alpha= 10^{-10}\sim10^{-14}$\cite{johnson2013accelerating}.  \par 
Using the convexity assumption of the loss function, and a  lemma in \cite{shalev2009stochastic}, we can also give a bound on the population excess risk, details are in supplemental material.

Corollary \ref{col1} and \ref{col2}  provide answers to our motivating question. That is, for loss functions which are $(8, T)$-smooth, we can get a lower sample complexity, if they are $(\infty,T)$-smooth, there is an $\epsilon$-LDP algorithm 
for empirical and population excess risks achieving error $\alpha$ with sample complexity which is independent from the dimensionality $p$ in the term $\alpha$. 
This result does not contradict the results by \citet{smith2017interaction}. Indeed, the example they provide whose  sample complexity must depend on $\alpha^{-\Omega(p)}$, to achieve the $\alpha$ error, is actually non-smooth. \par 
In our result of $(\infty, T)$-smooth case, like in the one by~\citet{smith2017interaction}, there is still a dependency of the sample complexity in the term $c^p$, for some constant $c$. Furthermore ours has also a dependency in the term $D_p$. There is still the question about what condition would allow a sample complexity independent from this term. We leave this question for future works and we focus instead on the efficiency and further applications of our method.
\section{More Efficient Algorithms}
\label{sec:efficient}
Algorithm \ref{alg:2} has  computational time and communication complexity for each player which is exponential in the dimensionality. This is clearly problematic for every realistic practical application. For this reason, in this section, we study more efficient algorithms. In order for convenience, in this part we only focus on the case of $(\infty, T)$-smooth loss functions, it can be easily extended to general cases.

Consider the following lemma, showing an $\epsilon$-LDP algorithm for computing $p$-dimensional average (notice the extra conditions on $n$ and $p$ compared with Lemma \ref{lemma:1}).
\begin{lemma}\cite{DBLP:journals/corr/NissimS17aa}\label{lemma:3}
	Consider player  $i\in [n]$ holding data $v_i\in \mathbb{R}^p$ with coordinate between $0$ and $b$.
 Then for $0<\beta<1,\, 0<\epsilon$ such that $n\geq 8p\log (\frac{8p}{\beta})$ and $\sqrt{n}\geq \frac{12}{\epsilon}\sqrt{\log \frac{32}{\beta}}$, there is an $\epsilon$-LDP algorithm, LDP-AVG, with probability at least $1-\beta$,  the output $a\in \mathbb{R}^p$ satisfying: 
	$\max_{j\in[d]}|a_j-\frac{1}{n}\sum_{i=1}^{n}[v_i]_j|\leq O(\frac{bp}{\sqrt{n}\epsilon}\sqrt{\log \frac{p}{\beta}})$.
	Moreover, the computation cost for each user is $O(1)$\footnote{Note that here we use an weak version of their result}.
\end{lemma}
By using this lemma and by discretizing the grid with some interval steps , we can design an algorithm which requires $O(1)$ computation time and $O(\log n)$-bits communication per player (see Appendix). However, we would like to do even better and obtain constant communication complexity.
Instead of discretizing the grid, we apply a technique,  firstly proposed by \citet{bassily2015local}, which permits to transform any `sampling resilient' $\epsilon$-LDP protocol into a protocol with 1-bit communication complexity. 
Roughly speaking, a protocol is sampling resilient if its output on any dataset $S$ can be approximated well by its output on a random subset of half of the players. 

Since our algorithm only uses the LDP-AVG protocol, we can show that it is indeed sampling resilient. Inspired by this result, we propose Algorithm~\ref{alg1:3} and obtain the following theorem. 
\begin{theorem}\label{thm:5}
	For $\epsilon\leq\ln 2$ and $0<\beta<1$, Algorithm \ref{alg1:3} is $\epsilon$-LDP. If the loss function $\ell(\cdot,x)$ is $(\infty,T)$-smooth for all $x\in \mathcal{D}$ and 
	$n=\Omega(\max\{\frac{\log \frac{1}{\beta}4^{p(p+1)}}{\epsilon^2 D_{p}^2}, p(k+1)^p\log (k+1), \frac{1}{\epsilon^2}\log \frac{1}{\beta}\})$ for some constant $c$, then 
	by setting
	$k=O\big((\frac{D_p\sqrt{pn}\epsilon}{2^{(p+1)p}\sqrt{\log \frac{1}{\beta}}})^{\frac{1}{2p}}\big)$,  the results in Corollary \ref{col2} hold with probability at least $1-4\beta$.
 Moreover, for each player the time complexity is $O(1)$, and the communication complexity is $1$-bit.
\end{theorem}
\begin{algorithm}
	\caption{Player-Efficient Local Bernstein Mechanism with 1-bit communication per player}
	\label{alg1:3}
	\begin{algorithmic}[1]
		\State {\bfseries Input:} Player $i\in [n]$ holding data $x_i\in \mathcal{D}$, public loss function $\ell:[0,1]^p \times \mathcal{D}\mapsto [0,1]$, privacy parameter $\epsilon\leq \ln 2$, and parameter $k$.
		\State  {\bfseries Preprocessing:} 
		\State Generate $n$ independent public strings\\ $y_1=\text{Lap}(\frac{1}{\epsilon}), \cdots, y_n=\text{Lap}(\frac{1}{\epsilon})$.
		\State Construct the grid $\mathcal{T}=\{\frac{v_1}{k},\ldots,\frac{v_p}{k}\}_{\{v_1,\ldots,v_p\}}$, where $\{v_1,\ldots,v_p\}\in\{0,1,\cdots,k\}^p$.
		\State Partition randomly $[n]$ into $d=(k+1)^p$ subsets $I_1,I_2,\cdots,I_d$, and associate each $I_j$ to a grid point $\mathcal{T}(j)\in \mathcal{T}$.	
		\For{Each Player $i\in[n]$}
		\State
		Find  $I_{l}$ such that $i\in  I_{l}$. Calculate $v_i=\ell(\mathcal{T}(l);x_i)$.
		\State Compute $p_i=\frac{1}{2}\frac{\text{Pr}[v_i+\text{Lap}(\frac{1}{\epsilon})=y_i]}{\text{Pr}[\text{Lap}(\frac{1}{\epsilon})=y_i]}$
		\State Sample a bit $b_i$ from $\text{Bernoulli}(p_i)$ and send it to the server.
		\EndFor
		\For{The Server}
		\For {$i=1\cdots n$}
		\State Check if $b_i=1$, set $\tilde{z_i}=y_i$, otherwise $\tilde{z_i}=0$.
		\EndFor
		\For {each $l\in[d]$}
		\State Compute $v_{\ell}=\frac{n}{|I_{l}|}\sum_{i\in I_{\ell}}\tilde{z_i}$
		\State Denote the corresponding grid point $(\frac{v_1}{k},\ldots,\frac{v_p}{k})\in \mathcal{T}$ of $I_l$, then denote $\hat{L}((\frac{v_1}{k},\cdots,\frac{v_p}{k});D)=v_{l}$.
		\EndFor
		\State Construct Bernstein polynomial for the perturbed empirical loss $\tilde{L}$ as in Algorithm \ref{alg:2}. Denote $\tilde{L}(\cdot,D)$  the corresponding function. 
		\State Compute $\theta_{\text{priv}}=\arg\min_{\theta\in \mathcal{C}}\tilde{L}(\theta;D)$.
		\EndFor
	\end{algorithmic}
\end{algorithm}

Now we study the algorithm from the server's complexity perspective. The polynomial construction time complexity is $O(n)$, where the most inefficient part is finding $\theta_{\text{priv}}=\arg\min_{\theta\in \mathcal{C}}\tilde{L}(\theta,D)$. In fact, this function may be  non-convex;  but unlike general non-convex functions, it can be $\alpha$-uniformly approximated
by a 

convex function $\hat{L}(\cdot;D)$ if the loss function is convex (by the proof of Theorem \ref{theorem:3}), although we do not have access to it.

Thus, we can see this problem as an instance of Approximately-Convex Optimization, which has been studied recently  by \citet{risteski2016algorithms}.

\begin{definition}\cite{risteski2016algorithms}\label{def:7}
	We say that a convex set $\mathcal{C}$ is $\mu$-well-conditioned for $\mu\geq 1$, if there exists a function $F:\mathbb{R}^p\mapsto \mathbb{R}$ such that $\mathcal{C}=\{x|F(x)\leq 0\}$ and for every $x\in \partial K: \frac{\|\nabla^2F(x)\|_2}{\|\nabla F(x)\|_2}\leq \mu$.
\end{definition}

\begin{lemma}[Theorem 3.2 in \cite{risteski2016algorithms}]\label{lemma:4}
	Let $\epsilon,\Delta$ be two real numbers such that 
	$\Delta\leq \max\{\frac{\epsilon^2}{\mu\sqrt{p}},\frac{\epsilon}{p}\}\times \frac{1}{16348}$.
	Then, there exists an algorithm $\mathcal{A}$ such that for any given $\Delta$-approximate convex function $\tilde{f}$ over a $\mu$-well-conditioned convex set $\mathcal{C}\subseteq\mathbb{R}^p$ of diameter 1 (that is, there exists a 1-Lipschitz convex function $f:\mathcal{C}\mapsto \mathbb{R}$ such that for every $x\in \mathcal{C}, |f(x)-\tilde{f}(x)|\leq \Delta$),  $\mathcal{A}$ returns a point $\tilde{x}\in\mathcal{C}$ with probability at least $1-\delta$ in time $\text{Poly}(p,\frac{1}{\epsilon},\log \frac{1}{\delta})$ 
	and with the following guarantee 
	$\tilde{f}(\tilde{x})\leq \min_{x\in \mathcal{C}}\tilde{f}(x)+\epsilon$. 
\end{lemma}

Based on Lemma \ref{lemma:4} (for $\tilde{L}(\theta;D)$) and Corollary \ref{col2}, and taking $\epsilon=O(p\alpha)$, we have the following.

\begin{theorem}\label{thm:6}
     Under the conditions in Corollary \ref{col2}, and assuming that $n=\tilde{\Omega}(4^{p(p+1)}\log(1/\beta)D_p^2 p \epsilon^{-2}\alpha^{-4})$,  that the loss function $\ell(\cdot, x)$ is $1$-Lipschitz and convex for every $x\in \mathcal{D}$, that the constraint set $\mathcal{C}$ is convex and  $\|\mathcal{C}\|_2\leq 1$, and satisfies $\mu$-well-condition property (see Definition \ref{def:7}), if the error $\alpha$ satisfies $\alpha\leq C\frac{\mu}{p\sqrt{p}}$ for some universal constant $C$, then there is an algorithm $\mathcal{A}$ which runs in $\text{Poly}(n, \frac{1}{\alpha},\log \frac{1}{\beta})$ time\footnote{Note that since here we assume $n$ is at least exponential in $p$, thus the algorithm is not fully polynomial.}
      for the server, and with probability $1-2\beta$ the output $\tilde{\theta}_{\text{priv}}$ of $\mathcal{A}$ satisfies 
     $\tilde{L}(\tilde{\theta}_{\text{priv}};D)\leq \min_{\theta\in \mathcal{C}}\tilde{L}(\theta;D)+O(p\alpha),$
     which means that $\text{Err}_{D}(\tilde{\theta}_{\text{priv}})\leq O(p\alpha)$.
\end{theorem} 
Combining  with  Theorem \ref{thm:5}, \ref{thm:6} and Corollary \ref{col2}, and taking $\alpha=\frac{\alpha}{p}$, we have our final result:
\begin{theorem}\label{thm:7}
	Under the conditions of Corollary \ref{col2}, Theorem \ref{thm:5} and \ref{thm:6}, and for any $C\frac{\mu}{\sqrt{p}}>\alpha>0$, if we further set $n=\tilde{\Omega}(4^{p(p+1)}\log(1/\beta)D_p^2 p^5 \epsilon^{-2}\alpha^{-4})$, then there is an $\epsilon$-LDP algorithm, with $O(1)$ running time and $1$-bit communication per player, and $\text{Poly}(\frac{1}{\alpha},\log \frac{1}{\beta})$ running time for the server. Furthermore, with probability at least $1-5\beta$, the output $\tilde{\theta}_{\text{priv}}$ satisfies $\text{Err}_{D}(\tilde{\theta}_{\text{priv}})\leq O(\alpha)$.
\end{theorem}
Note that compared with the sample complexity in Theorem \ref{thm:7} and Corollary \ref{col2}, we have an additional factor of $p^4$; however, the $\alpha$ terms are the same. In fact, we could extend our method to other LDP problems, we study how to answer the class of k-way marginals and smooth queries under LDP, which has not been studied before. 
\section{LDP Algorithms for Learning K-way Marginals Queries and Smooth Queries By using Polynomial Approximation}

In this section, we will show further applications of our idea by giving $\epsilon$-LDP algorithms for answering sets of queries. All the queries we consider in this section are linear, that is, of the form $q_f(D)=\frac{1}{|D|}\sum_{x\in D}f(x)$ for some function $f$. It will be convenient to have a notion of accuracy for the algorithm we will present with respect to a set of queries. This is defined as follow:

\begin{definition}
	Let $\mathcal{Q}$ denote a set of queries. An algorithm $\mathcal{A}$ is said to have $(\alpha,\beta)$-accuracy for size $n$ databases with respect to $\mathcal{Q}$, if for every $n$-size dataset $D$, the following holds:
	$\text{Pr}[ \exists q\in \mathcal{Q}, |\mathcal{A}(D,q)-q(D)|\geq \alpha]\leq \beta.$
\end{definition}

\subsection{K-way Marginals Queries}

Now we consider a database $D=(\{0,1\}^p)^n$, where each row corresponds to an individuals record. A marginal query is specified by a set $S\subseteq [p]$ and a pattern $t\in \{0,1\}^{|S|}$. Each such query asks: `What fraction of the individuals in $D$ has each of the attributes set to $t_j$?'. We will consider here k-way marginals which are the subset of marginal queries specified by a set $S\subseteq[p]$ with $|S|\leq k$. K-way marginals permit to represent several statistics over datasets, including contingency tables, and the problem to release them under differential privacy has been studied extensively in the literature~\cite{hardt2012private,gupta2013privately,thaler2012faster,GaboardiAHRW14}. All these previous works have considered the central model of differential privacy, and only the recent work~\cite{Kulkarni17} studies this problem in the local model, while their methods are based Fourier Transform.  We now use the LDP version of Chebyshev polynomial approximation to give an efficient way of constructing a sanitizer for releasing k-way marginals. 

Since learning the class of $k$-way marginals is equivalent to learning the class of monotone k-way disjunctions \cite{hardt2012private}, we will only focus on the latter. The reason why we can locally privately learning them is that they form a $\mathcal{Q}$-Function Family.

\begin{definition}[$\mathcal{Q}$-Function Family]\label{def:10}
	Let $\mathcal{Q}=\{q_y\}_{y\in Y_{\mathcal{Q}}\subseteq \{0,1\}^m}$ be a set of counting queries on a data universe $\mathcal{D}$, where each query is indexed by an $m$-bit string. We define the index set of $\mathcal{Q}$ to be the set $Y_{\mathcal{Q}}=\{y\in \{0,1\}^m| q_y\in \mathcal{Q}\}$.\\
	We define a $\mathcal{Q}$-Function Family $\mathcal{F}_{\mathcal{Q}}=\{f_{\mathcal{Q},x}:\{0,1\}^m \mapsto \{0,1\}\}_{x\in\mathcal{D}}$ as follows: for every data record $x\in D$, the function $f_{\mathcal{Q},x}:\{0,1\}^m\mapsto \{0,1\}$ is defined as $f_{\mathcal{Q},x}(y)=q_y(x)$. Given a database $D\in \mathcal{D}^n$, we define $f_{\mathcal{Q},D}(y)=\frac{1}{n}\sum_{i=1}^{n}f_{\mathcal{Q},x^i}(y)=\frac{1}{n}\sum_{i=1}^{n}q_y(x^i)=q_y(D)$, where $x^i$ is the $i$-th row of $D$.
\end{definition}

This definition guarantees that $\mathcal{Q}$-function queries can be computed from their values on the individual's data $x^i$. We can now formally define the class of  monotone k-way disjunctions.

\begin{definition}\label{def:11}
	Let $\mathcal{D}=\{0,1\}^p$. The query set $\mathcal{Q}_{disj,k}=\{q_y\}_{y\in Y_k\subseteq \{0,1\}^p}$ of monotone $k$-way disjunctions over $\{0,1\}^p$ contains a query $q_y$ for every $y\in Y_k=\{y\in\{0,1\}^p| |y|\leq k\}$. Each query is defined as $q_y(x)= \vee_{j=1}^{p}y_jx_j$. The $\mathcal{Q}_{disj,k}$-function family $\mathcal{F}_{\mathcal{Q}_{disj,k}}=\{f_x\}_{x\in\{0,1\}^p}$ contains a function $f_x(y_1,y_2,\cdots,y_p)=\vee_{j=1}^{p}y_jx_j$ for each $x\in \{0,1\}^p$.
\end{definition}

Definition \ref{def:10} guarantees that if we can uniformly approximated the function $f_{\mathcal{Q},x}$ by polynomials $p_x$, then we can also have an approximation of $f_{\mathcal{Q},D}$, {\em i.e.} we can approximate $q_y(D)$ for every $y$ or 
all the queries in the class $\mathcal{Q}$. Thus, if we can locally privately estimate the sum of coefficients of the monomials for the $m$-multivariate functions $\{p_x\}_{x\in D}$, we can uniformly approximate $f_{\mathcal{Q},D}$. Clearly, this can be done by Lemma 2, if the coefficients of the approximated polynomial are bounded. 

In order to uniformly approximate the class $\mathcal{Q}_{disj,k}$, we use Chebyshev polynomials.
\begin{definition}[Chebyshev Polynomials]\label{def:9}
	For every $k\in \mathbb{N}$ and $\gamma>0$, there exists a univariate real polynomial $p_k(x)=\sum_{j=0}^{t_k}c_ix^i$ of degree $t_k$ such that 
 $t_k=O(\sqrt{k}\log(\frac{1}{\gamma}))$;
 for every $i\in [t_k], |c_i|\leq 2^{O(\sqrt{k}\log(\frac{1}{\gamma}))}$; and 
 $p(0)=0, |p_k(x)-1|\leq \gamma, \forall x\in[k]$.

\end{definition}
\begin{algorithm}[h]
	\caption{Local Chebyshev Mechanism for $\mathcal{Q}_{\text{disj,k}}$ }
	\label{Aalg:1}
	\begin{algorithmic}[1]
		\State{\bfseries Input:} Player $i\in [n]$ holding data $x_i\in \{0,1\}^p$, privacy parameter $\epsilon>0$, error bound $\alpha$, and $k\in \mathbb{N}$.
		\For{Each Player $i\in[n]$}
		\State
		Consider the $p$-multivariate polynomial $q_{x_i}(y_1,\ldots,y_p)= p_k(\sum_{j=1}^{p}y_j[x_i]_j)$, where $p_k$ is defined as in Definition \ref{def:9} with $\gamma=\frac{\alpha}{2}$.
		\State 
		Denote the coefficients of $q_{x_i}$ as a vector $\tilde{q}_{i}\in \mathbb{R}^{\binom{p+t_k}{t_k}}$(since there are $\binom{p+t_k}{t_k}$ coefficients in a $p$-variate polynomial with degree $t_k$), note that each $\tilde{q}_{i}$ can bee seen as a $p$-multivariate polynomial $q_{x_i}(y)$.
		\EndFor
		\For{The Server}
		\State Run LDP-AVG from Lemma 1 on $\{\tilde{q}_i\}_{i=1}^{n}\in\mathbb{R}^{\binom{p+t_k}{t_k}}$ with parameter $\epsilon$, $b=p^{O(\sqrt{k}\log(\frac{1}{\gamma}))}$, denote the output as $\tilde{p}_D\in \mathbb{R}^{\binom{p+t_k}{t_k}}$, note that $\tilde{p}_D$ also corresponds to a $p$-multivariate polynomial.
		\State
		For each query $y$ in $\mathcal{Q}_{\text{disj,k}}$ (seen as a $d$ dimension vector), compute the $p$-multivariate polynomial $\tilde{p}_D(y_1,\ldots,y_p)$.
		\EndFor
	\end{algorithmic}
\end{algorithm}

\begin{lemma}\cite{thaler2012faster}\label{Alemma:1}
	For every $k,p \in \mathbb{N}$, such that $k\leq p$, and every $\gamma>0$, there is a family of $p$-multivariate polynomials of degree $t=O(\sqrt{k}\log(\frac{1}{\gamma}))$with  coefficients bounded by $T=p^{O(\sqrt{k}\log(\frac{1}{\gamma}))}$, which uniformly approximate the family $\mathcal{F}_{\mathcal{Q}_{\text{disj,k}}}$ over the set $Y_k$ (Definition \ref{def:11}) with error bound $\gamma$. That is, there is a family of polynomials $\mathcal{P}$ such that for every $f_x\in\mathcal{F}_{\mathcal{Q}_{\text{disj,k}}}$, there is $p_x\in \mathcal{P}$ which satisfies $\sup_{y\in Y_k}|p_x(y)-f_x(y)|\leq \gamma$.
\end{lemma}

By combining the ideas discussed above and Lemma \ref{Alemma:1}, we have Algorithm \ref{Aalg:1} and the following theorem.

\begin{theorem}
For $\epsilon >0$ Algorithm \ref{Aalg:1} is $\epsilon$-LDP. Also, for $0<\beta<1$, there are constants $C, C_1$ such that for every $k,p,n\in\mathbb{N}$ with $k\leq p$, if $n\geq \Omega(\max\{\frac{p^{C\sqrt{k}\log \frac{1}{\alpha}}\log \frac{1}{\beta}}{\epsilon^2\alpha^2}, \frac{\log \frac{1}{\beta}}{\epsilon^2},p^{C_1\sqrt{k}\log \frac{1}{\alpha}}\log \frac{1}{\beta}\})$,  this algorithm is $(\alpha,\beta)$-accuracy with respect to $\mathcal{Q}_{\text{disj},k}$. The running time for player is $\text{Poly}(p^{O(\sqrt{k}\log\frac{1}{\alpha})})$, and the running time for server is at most $O(n)$ and the time for answering a query is $O(p^{C_2\sqrt{k}\log \frac{1}{\alpha}})$ for some constant $C_2$.
Moreover, as in Section 5, the communication complexity can be improved to 1-bit per player.
 \end{theorem}

\subsection{Smooth Queries}
We now consider the case where each player $i\in[n]$ holds a data $x_i\in \mathbb{R}^p$ and we want to estimate the kernel density for a given point $x_0\in\mathbb{R}^p$. 
A natural question is: If we want to estimate Gaussian kernel density of a given point $x_0$ with many different bandwidths, can we do it simultaneously under $\epsilon$ local differential privacy?

We can see this kind of queries as a subclass of the smooth queries. So, like in the case of k-way marginals queries, we will give an $\epsilon$-LDP sanitizer for smooth queries. 
Now we consider the data universe $\mathcal{D}=[-1,1]^p$, and  dataset $D\in \mathcal{D}^n$. For a positive integer $h$ and constant $T>0$, we denote the set of all $p$-dimensional $(h,T)$-smooth function (Definition \ref{def:5}) as $C^{h}_T$, and $\mathcal{Q}_{C^{h}_T}=\{q_f(D)=\frac{1}{n}\sum_{x\in D}f(D), f\in C^{h}_T\}$ the corresponding set of queries.  The idea of the algorithm is similar to the one used for the k-way marginals; but instead of using Chebyshev polynomials, we will use trigonometric polynomials. We now assume that the dimensionality $p$, $h$ and $T$ are constants so all the result in big $O$ notation will be omitted. The idea of Algorithm \ref{alg:7} is actually based on the following Lemma.

\begin{lemma}\cite{wang2016differentially}\label{Alemma:2}
	Assume $\gamma>0$. For every $f\in C^h_T$, defined on $[-1,1]^p$, let $g_f(\theta_1,\ldots,\theta_p)=f(\cos(\theta_1),\ldots,\cos(\theta_p))$, for $\theta_i\in [-\pi,\pi]$. Then there is an even trigonometric polynomial $p$ whose degree for each variable is $t(\gamma)=(\frac{1}{\gamma})^{\frac{1}{h}}$:
	\begin{equation}\label{Aeq:2}
	p(\theta_1,\ldots,\theta_p)=\sum_{0\leq r_1,\ldots,r_p< t(\gamma)}c_{r_1,\ldots,r_p
		}\prod_{i=1}^{p}\cos(r_i\theta_i),
	\end{equation}
such that
1) $p$ $\gamma$-uniformly approximates $g_f$, i.e. $\sup_{x\in[-\pi,\pi]^p}|p(x)-g_f(x)|\leq \gamma.$
2) The coefficients are uniformly bounded by a constant $M$ which only depends on $h, T$ and  $p$.
3)  Moreover, the whole set of the coefficients can be computed in time $O\big((\frac{1}{\gamma})^{\frac{p+2}{h}+\frac{2p}{h^2}}\text{poly}\log \frac{1}{\gamma})\big)$.
\end{lemma}
By (\ref{Aeq:2}), we can see that all the $p(x)$ which corresponds to $g_f(x)$, representing functions $f\in C_T^{h}$, have the same basis $\prod_{i=1}^{p}\cos(r_i\theta_i)$. So, we can use Lemma 1 or 2 to estimate the average of the basis. Then, for each query $f$ the server can only compute the corresponding coefficients $\{c_{r_1,r_2,\cdots,r_p}\}$. This idea is implemented in  Algorithm \ref{Aalg:2}  for which we have the following result.
\begin{theorem}
For $\epsilon>0$, Algorithm \ref{Aalg:2} is $\epsilon$-LDP. Also for $\alpha>0$, $0<\beta<1$, if $n\geq \Omega(\max \{\log^{\frac{5p+2h}{2h}}(\frac{1}{\beta})\epsilon^{-2}\alpha^{-\frac{5p+2h}{h}}, \frac{1}{\epsilon^2}\log(\frac{1}{\beta})\})$ and  $t=O((\sqrt{n}\epsilon)^{\frac{2}{5p+2h}})$, then Algorithm \ref{Aalg:2} is $(\alpha,\beta)$-accurate with respect to $\mathcal{Q}_{C^h_T}$.  The time for answering each query is $\tilde{O}((\sqrt{n}\epsilon)^{\frac{4p+4}{5p+2h}+\frac{4p}{5ph+2h^2}})$, where $O$ omits $h,T,p$ and some $\log$ terms.
For each player, the computation and communication cost could be improved to $O(1)$ and 1 bit, respectively, as in Section 5.
\end{theorem}
\vspace{-0.05in}
\begin{algorithm}[h]
	\caption{Local Trigonometry Mechanism for $\mathcal{Q}_{C^{h}_T}$ }
	\label{Aalg:2}
	\begin{algorithmic}[1]
		\State {\bfseries Input:} Player $i\in [n]$ holding data $x_i\in [-1,1]^p$, privacy parameter $\epsilon>0$, error bound $\alpha$, and $t\in \mathbb{N}$.
		$\mathcal{T}_{t}^p=\{0,1,\cdots,t-1\}^p$. For a vector $x=(x_1,\ldots,x_p)\in [-1,1]^p$, denote operators $\theta_i(x)=\arccos(x_i), i\in[p]$.
		\For{Each Player $i\in[n]$}
		\For{Each $v=(v_1,v_2,\cdots,v_p)\in \mathcal{T}_{t}^p$}
		\State Compute $p_{i;v}=\cos(v_1\theta_1(x_i))\cdots \cos(v_p\theta_p(x_i))$
		\EndFor
		\State 
                Let $p_i=(p_{i;v})_{v\in \mathcal{T}_{t}^p}$.
		\EndFor
		\For{The Server}
		\State Run LDP-AVG from Lemma 1 on $\{p_i\}_{i=1}^{n}\in\mathbb{R}^{t^p}$ with parameter $\epsilon$, $b=1$, denote the output  as $\tilde{p}_D$.
		\State For each query $q_f\in \mathcal{Q}_{C^h_T}$. Let $g_f(\theta)=f(\cos(\theta_1),\cos(\theta_2),\cdots,\cos(\theta_p))$.
		\State Compute the trigonometric polynomial approximation $p_t(\theta)$ of $g_f(\theta)$, where 
		$p_t(\theta)=\sum_{r=(r_1,r_2\cdots r_p),\|r\|_{\infty}\leq t-1}c_r\cos(r_1\theta_1)\cdots \cos(r_p\theta_p)$ as in (\ref{Aeq:2}).
		Denote the vector of the coefficients $c\in \mathbb{R}^{t^p}$.
		\State Compute $\tilde{p}_D\cdot c$.
		\EndFor	
	\end{algorithmic}
\end{algorithm}

\section{High Dimensional Case}

In the previous parts and \cite{smith2017interaction}, it is always assume that $n\geq p$ (while ours need $\log n\geq O(p)$). However, many problem in machine learning are in high dimension space {\em i.e.} $n\ll p$. We will show a general method for this case if the loss function is generalized linear function. Due to the space limit, all the definitions and general statements are in Appendix.\par 
A function $\ell(w, x)$ is called generalized linear function \cite{shalev2009stochastic} if $\ell(w,x)=f(\langle w,y \rangle, z)$ for $x=(y, z)$, where $y\in\mathbb{R}^p$ is the data and $z$ is the label, actually, many loss functions satisfies the condition, such as Logistic Regression, Hinge loss, linear Regression etc. We will assume the dataset satisfies $\|y_i\|\leq 1, \|z_i\|\leq 1$ for all $i\in[n]$. Also we will assume that $f$ is $1$-Lipschitz convex in the first argument, $\|\mathcal{C}\|_2\leq 1$ follows \cite{smith2017interaction} and is isotropic\footnote{A convex set is isotropic if a random vector chosen uniformly from $\mathcal{K}$ according to the volume is isotropic. A random vector $a$ is isotropic if for all $b\in\mathbb{R}^p, \mathbb{E}[\langle a,b \rangle^2]=\|b\|^2$, such as polytope.}.\par 
The motivation of our algorithm is inspired by \cite{kasiviswanathan2016efficient}, that is, we firstly do dimension reduction for each data $y_i$, that is $D'=\{(\Phi y_1,z_1),\cdots, (\Phi y_n, z_n)\}$, where $\Phi\in \mathbb{R}^{m\times p}$. Then we run a modified version of the algorithm in \cite{smith2017interaction}. After getting the private estimator $\bar{w}\in \mathbb{R}^m$, we use compress sensing technique by solving a optimization problem \cite{vershynin2015estimation} to recover $w^{\text{priv}}\in \mathbb{R}^p$. \par 
We have to note that we cannot use the $\epsilon$-LDP algorithm (see Figure 5 in \cite{smith2017interaction}) in \cite{smith2017interaction} since it needs $n\geq k$, where $k=O\big(\frac{2^{\frac{p-1}{2}}\sqrt{p}}{\alpha^{p-1}}\big)$, and 
$\alpha=O\big(
(\frac{\sqrt{p}}{\epsilon^2n}\log^3(\epsilon^2n))^{\frac{1}{p+1}}\big)$. 
This is said that $n\geq O(c^p)$, which is contradictory with our assumption. Actually, we will provide a similar algorithm which can remove this assumption. The idea is comes from \cite{DBLP:journals/corr/abs-1711-04740}, which shows that in non-interactive local model, every $(\epsilon, \delta)$-LDP protocal can be transformed as an $\epsilon$-LDP algorithm. Thus, our idea is, in Figure 5 of \cite{smith2017interaction},
instead of partitioning the dataset for $k$ parts and running the subroutine of Figure 1 in \cite{smith2017interaction}, here we will run $k$ directions for the whole dataset, by the advanced composition theorem (corollary 3.21 in \cite{dwork2014algorithmic}), if for each direction, we run $(\epsilon_0=O(\frac{\epsilon}{\sqrt{k\log(1/\delta)}},0)$-LDP, then the whole LDP algorithm is $(\epsilon,\delta)$-LDP, after that, we use the protocal in \cite{DBLP:journals/corr/abs-1711-04740} to transform the $(\epsilon, \delta)$-LDP algorithm to $O(\epsilon)$-DP. See Algorithm \ref{Ealg:6}.
\begin{algorithm*}[h]
	\caption{$(\epsilon, \delta)$ protocol LDP Algorithm}
	\label{Ealg:6}
	\begin{algorithmic}[1]
		\State{\bfseries Input:} Each user $i\in [n]$ has data $x_i\in \mathcal{D}$, privacy parameters $\epsilon, \delta$, public loss function $\ell:[0,1]^p \times \mathcal{D}\mapsto [0,1]$ satisfies the assumption in \cite{smith2017interaction}, and parameter $k$( we will specify it later). 
		\State {\bfseries Preprocessing:} 
		\State Choose $k$ random directions, $u_1, u_2, \cdots, u_k$ and send to each user.
		\For{Each user $i\in[n]$}
		\State For each $j\in [k]$, invoke 1D-General (Figure 3 in \cite{smith2017interaction}) with $(x_i, u_j)$ with $\epsilon=\frac{\epsilon}{2\sqrt{2k\log(1/\delta)}}$ and $\gamma=\gamma/k$, output $\mathcal{T}_{i,j}$. Then send $\mathcal{T}_i=(\mathcal{T}_{i,1}, \cdots, \mathcal{T}_{i,k})$ to the server.
		\EndFor
		\For{The server}
		\State After receiving $\{\mathcal{T}_i\}_{i=1}^n$, do the following steps
		\State For $j\in [k]$, invokes 1-D General(Figure 4 in \cite{smith2017interaction}) with $\{\mathcal{T}_{i,j}\}_{i=1}^{n}$ to get $\hat{f}^j$.
		\State Compute $\theta_j=\arg\min_{\theta||u_j}\hat{f}^j$ and then compute $\theta_{\text{priv}}=\arg\min_j\hat{f}^j(\theta_j)$, output $\theta_{\text{priv}}$. 
		\EndFor
	\end{algorithmic}
\end{algorithm*}
We have the following theorem for Algorithm \ref{Ealg:6}, the proof is the same as in \cite{smith2017interaction}:
\begin{theorem}\label{atheorem:1}
Under the same assumption as in Theorem 10 of \cite{smith2017interaction}. Algorithm \ref{Ealg:6} is $(\epsilon, \delta)$-LDP for any $1>\epsilon>0, 0<\delta<1$. Also, we have for every $k$, with probability at least $1-\gamma$, the output satisfies 
\begin{equation}
\|\hat{f}^j-L_{\mathcal{P}}\|_{\infty}\leq O\big(\frac{\log(\epsilon^2 n/k\log(1/\delta))}{\epsilon}
\sqrt{\frac{k\log(1/\delta)\log(\epsilon^2n/\log(1/\delta))\gamma}{n}}\big).
\end{equation}
Furthermore, if we take $k=O\big(\frac{2^{(p-1)/2}\log(1/\gamma)}{\alpha^{p-1}}\sqrt{\frac{\pi p}{2}}\big)$ where $\alpha=O\big(
(\frac{\sqrt{p}}{\epsilon^2n}\log^3(\epsilon^2n)\log^2(1/\gamma))^{\frac{1}{p+1}}\big)$, here $O$ omits $\log(1/\delta)$ factors. Then we have  $\|\hat{f}-L_{\mathcal{P}}\|_{\infty}\leq \tilde{O}(\alpha)$, with probability at least $1-2\gamma$.
\end{theorem}
Now, we have almost the same upper bound as in Theorem 10 of \cite{smith2017interaction}. Then after using GenProt in \cite{DBLP:journals/corr/abs-1711-04740}, we can have an $10\epsilon$-LDP which has the same error bound as in Theorem \ref{atheorem:1}:
\begin{theorem}\label{atheorem:2}
Let $\epsilon\leq \frac{1}{4}$, if we set $\delta=O(\frac{\epsilon\gamma}{n\ln(2n/\gamma)})$ in Algorithm \ref{Ealg:6} as the protocol and run the Genprot algorithm in \cite{DBLP:journals/corr/abs-1711-04740}. Then there is an $10\epsilon$-LDP algorithm, such that whith probability ast least $1-3\gamma$, the output $w_{\text{priv}}$ satisfies
\begin{equation*}
	 \text{Err}_{\mathcal{P}}(\theta_{\text{priv}})\leq \tilde{O}\big((
	\frac{\sqrt{p}\log^2(1/\beta)}{\epsilon^2 n}
	)^{\frac{1}{p+1}}\big).
\end{equation*}

\end{theorem}
Our method is based on the following lemma in \cite{dirksen2016dimensionality}.
\begin{algorithm}[h]
	\caption{DR-ERM-LDP}
	\label{alg:9}
	\begin{algorithmic}[1]
		\State {\bfseries Input:} Player $i\in [n]$ holding data $x_i=(y_i,z_i)\in \mathcal{D}$, where $\|y_i\|\leq 1$, privacy parameter $\epsilon$.
		\State The server generate an random sub-Gaussian matrix $\Phi\in \mathbb{R}^{m\times p}$ in Lemma \ref{lemma:4}, and send the seed of this random matrix to all players.
		\For{Each Player $i$}
		\State
		Calculate $x_i'=(\Phi y_i,z_i)$
		\State 
		Run the modified $\epsilon$-local DP algorithm of \cite{smith2017interaction} for $D'=\{z_i'\}$ with constrained set $\mathcal{C}=\Phi\mathcal{C}$ and loss function $f$. The server get the output as $\bar{w}\in\mathbb{R}^m$.
		\EndFor
		\State The server solving the following problem 
		$w^{\text{priv}}=\arg \min_{w\in\mathbb{R}^p}\|w\|_{\mathcal{C}}$ subject to $\Phi w=\bar{w}$.
	\end{algorithmic}
\end{algorithm}

\begin{lemma}\label{lemma:4}
Let $\tilde{\Phi}\in \mathbb{R}^{m\times p}$ be an random matrix, whose rows are i.i.d mean-zero, isotropic, subgaussian random variable in $\mathbb{R}^d$ with $\psi=\|\Phi_i\|_{\psi_2}$. Let $\Phi=\frac{1}{\sqrt{m}}\tilde{\Phi}$. let $S$ be a set if points in $R^d$. Then there is a constant $C>0$ such that for any $0<\gamma,\beta<1$. $
\text{Pr}[\sup_{a\in S}|\|\Phi a\|^2-\|a\|^2\leq \gamma\|a\|^2]\leq \beta,
$
provided that $m\geq \frac{C\psi^4}{\gamma^2}\max\{\mathcal{G}_{\mathcal{S}},\log(1/\beta)\}^2$.
\end{lemma}
\begin{theorem}\label{theorem:7}
Under the assumption above. For any $\epsilon\leq \frac{1}{4}$, Algorithm \ref{alg:9} is $O(\epsilon)$-LDP. Moreover, setting $m=\Theta(\frac{\psi^4(\mathcal{G}_{\mathcal{C}}+\sqrt{\log n})^2\log(n/\beta)}{\gamma^2})$, where $\gamma=\Theta(\frac{\psi\sqrt{(\mathcal{G}_{\mathcal{C}}+\sqrt{\log n})}\log(1/\beta)\sqrt[4]{\log(n/\beta)}}{\sqrt{n}\epsilon})$. Then with probability at least $1-\beta$, 
$$\text{Err}_D(w^{\text{priv}})=\tilde{O}\big(\big(\frac{\log(1/\beta)\psi\sqrt{(\mathcal{G}_{\mathcal{C}}+\sqrt{\log n})}\sqrt[4]{\log(n/\beta)}}{\sqrt{n}\epsilon}\big)^{\frac{1}{1+m}}),$$
where $\psi$ is the subgaussian norm of the distribution of $\Phi$, $\mathcal{G}_{\mathcal{C}}$ is the Gaussian width of $\mathcal{C}$.
\end{theorem}
\begin{corollary}
If $\Phi$ is a standard Gaussian random matrix, $\mathcal{C}$ is the $\ell_1$ norm ball $B^p_1$ or the distribution simplex in $\mathbb{R}^p$, and $n\ll p\leq e^{cn}$ for some constant $c$. Then we have the bound in Theorem 7 is just $\tilde{O}\big(\big(\frac{\log(1/\beta)\sqrt[4]{\log p}\sqrt[4]{\log(n/\beta)}}{\sqrt{n}\epsilon}\big)^{\frac{1}{1+m}}\big)$, where $m=O(n\epsilon^2\log p\sqrt{\log(n/\beta)})$. We can see the bound in Theorem \ref{theorem:7} is always better than the bound in Theorem \ref{theorem:1} since ours is always less than $O(1)$.
\end{corollary}
\section{Conclusion and Discussion}
In this paper, we studied ERM under non-interactive LDP and proposed an algorithm which is based on Bernstein polynomial approximation. We showed that if the loss function is smooth enough, then the sample complexity to achieve $\alpha$ error is $\alpha^{-c}$ for some positive constant $c$, which improves significantly on the previous result  of $\alpha^{-(p+1)}$. Moreover, we proposed efficient algorithms for both player and server views. We also showed how a similar idea based on other polynomial approximations can be used to answering k-way-marginals and smooth queries in the local model.\par

In our algorithms the sample complexity still depends on the dimension $p$, in the term of $c^{p}$ for constant $c$. We will focus on removing this dependency in future work. Additionally, we will study the difference between strongly convex and convex loss functions in the non-interactive LDP setting.
\bibliographystyle{plainnat}
\bibliography{icml}
\newpage
\appendix
\section{Details in Section 3}
\begin{lemma}\cite{DBLP:journals/corr/NissimS17aa}\label{Blemma:1}
	Suppose that $x_1,\cdots,x_n$ are i.i.d sampled from $\text{Lap}(\frac{1}{\epsilon})$. Then for every $0\leq t<\frac{2n}{\epsilon}$, we have 
	\begin{equation*}
	\text{Pr}(|\sum_{i=1}^{n}x_i|\geq t)\leq 2\exp(-\frac{\epsilon^2t^2}{4n}).
	\end{equation*}
\end{lemma}
\begin{proof}[Proof of Lemma 1]
	Consider Algorithm 1. We have $|a-\frac{1}{n}\sum_{i=1}^{n}v_i|=|\frac{\sum_{i=1}^{n}x_i}{n}|$, where $x_i\sim \text{Lap}(\frac{b}{\epsilon})$. Taking $t=\frac{2\sqrt{n}\sqrt{\log \frac{2}{\beta}}}{\epsilon}$ and applying the above lemma, we prove the lemma. 
\end{proof}

\section{Details in Section 4}
\subsection{Proof of Theorem 3}
\begin{proof}[Proof of Theorem 3]
	The proof of the $\epsilon$-LDP comes from Lemma 1 and composition theorem. W.l.o.g, we assume T=1.
	To prove the theorem, it is sufficient to estimate $\sup_{\theta\in \mathcal{C}}|\tilde{L}(\theta;D)-\hat{L}(\theta;D)|\leq \alpha$ for some $\alpha$, since if it is true, denote $\theta^*=\arg\min_{\theta\in \mathcal{C}}\hat{L}(\theta;D)$, we have $\hat{L}(\theta_{\text{priv}};D)-\hat{L}(\theta^*;D)\leq \hat{L}(\theta_{\text{priv}};D)-\tilde{L}(\theta_{\text{priv}};D)+\tilde{L}(\theta_{\text{priv}};D)-\tilde{L}(\theta^*;D)+\tilde{L}(\theta^*;D)-\hat{L}(\theta^*;D)\leq \hat{L}(\theta_{\text{priv}};D)-\tilde{L}(\theta_{\text{priv}};D)+\tilde{L}(\theta^*;D)-\hat{L}(\theta^*;D)\leq 2\alpha$.\par
	Since we have $\sup_{\theta\in \mathcal{C}}|\tilde{L}(\theta;D)-\hat{L}(\theta;D)|\leq \sup_{\theta\in \mathcal{C}}|\tilde{L}(\theta;D)-B_k^{(h)}(\hat{L},\theta)|+\sup_{\theta\in \mathcal{C}}|B_k^{(h)}(\hat{L},\theta)-\hat{L}(\theta;D)|$. The second term is bounded by $O(D_h p\frac{1}{k^h})$ by Theorem 2.\par 
	For the First term, by (2) and the algorithm, we have 
	\begin{equation}
	\sup_{\theta\in \mathcal{C}}|\tilde{L}(\theta;D)-B_k^{(h)}(\hat{L},\theta)|\leq  \max_{v\in \mathcal{T}}|\tilde{L}(v;D)-\hat{L}(v;D)|
	\sup_{\theta\in \mathcal{C}}\sum_{j=1}^{p}\sum_{v_j=0}^{k}|\prod_{i=1}^{p}b_{v_i,k}^{(h)}(\theta_i)|.
	\end{equation}
	By Proposition 4 in \cite{alda2017bernstein}, we have $\sum_{j=1}^{p}\sum_{v_j=0}^{k}|\prod_{i=1}^{p}b_{v_i,k}^{(h)}(\theta_i)|\leq (2^h-1)^p$. Next lemma bounds the term $\max_{v\in \mathcal{T}}|\tilde{L}(v;D)-\hat{L}(v;D)|$, which is obtained by Lemma 1.
	
	\begin{lemma}
		If  $0<\beta<1, k$ and $n$ satisfy that $n\geq p\log(2/\beta)\log(k+1)$, then with probability at least $1-\beta$, for each $v\in \mathcal{T}$,
		\begin{equation}
		|\tilde{L}(v;D)-\hat{L}(v;D)|\leq O(\frac{\sqrt{\log \frac{1}{\beta}}\sqrt{p}\sqrt{\log (k)}(k+1)^p}{\sqrt{n}\epsilon}).
		\end{equation} 
	\end{lemma}
	
	\begin{proof}
		By Lemma 1, for a fixed $v\in \mathcal{T}$, if $n\geq \log \frac{2}{\beta}$, we have with probability $1-\beta$, $|\tilde{L}(v;D)-\hat{L}(v;D)|\leq \frac{2\sqrt{\log \frac{2}{\beta}}}{\sqrt{n}\epsilon}$. Taking the union of all $v\in \mathcal{T}$ and then taking $\beta=\frac{\beta}{(k+1)^p}$ (since there are $(k+1)^p$ elements in $\mathcal{T}$) and $\epsilon=\frac{\epsilon}{(k+1)^p}$, we get the proof.
	\end{proof}
	By $(k+1)< 2k$, we have 
	\begin{equation}\label{eq:8}
	\sup_{\theta\in \mathcal{C}}|\tilde{L}(\theta;D)-\hat{L}(\theta;D)|\leq O(\frac{D_hp}{k^h}+ \frac{2^{(h+1)p}\sqrt{\log\frac{1}{\beta}}\sqrt{ p\log k}k^p}{\sqrt{n}\epsilon}).
	\end{equation}
	Now we take $k=O(\frac{D_h\sqrt{pn}\epsilon}{2^{(h+1)p}\sqrt{\log \frac{1}{\beta}}})^{\frac{1}{h+p}}$. Since $n=\Omega(\frac{4^{p(h+1)}}{\epsilon^2 p D_{h}^2})$, we have $\log k>1$. Pluggning  it into (\ref{eq:8}), we  get
	\begin{equation}
	\sup_{\theta\in \mathcal{C}}|\tilde{L}(\theta;D)-\hat{L}(\theta;D)|\leq 	\tilde{O}(\frac{\log^{\frac{h}{2(h+p)}} (\frac{1}{\beta}) D_h^{\frac{p}{p+h}}p^{\frac{1}{2}+\frac{p}{2(h+p)}}2^{(h+1)p\frac{h}{h+p}}}{\sqrt{h+p}n^{\frac{h}{2(h+p)}}\epsilon^{\frac{h}{h+p}}})=	\tilde{O}(\frac{\log^{\frac{h}{2(h+p)}} (\frac{1}{\beta}) D_h^{\frac{p}{p+h}}p^{\frac{p}{2(h+p)}}2^{(h+1)p}}{n^{\frac{h}{2(h+p)}}\epsilon^{\frac{h}{h+p}}}).
	\end{equation}
	Also we can see that $n\geq p\log(2/\beta)\log(k+1)$ is true for $n=\Omega(\frac{4^{p(h+1)}}{\epsilon^2 p D_{h}^2})$.  Thus, the theorem follows. 
\end{proof}

\begin{proof}[Proof of Corollary 1 and 2]	
	Since the loss function is $(\infty,T)$-smooth, it is $(2p,T)$-smooth for all $p$. Thus, taking $h=p$ in Theorem 3, we get the proof.
\end{proof}
\subsection{Population Risk of Algorithm 2}
Here we will only show the case of $(\infty, T)$, it is the same for the general case.
\begin{theorem}\label{thm:4}
	Under the conditions in Corollary 2, if we further assume the loss function $\ell(\cdot,x)$ to be convex and $1$-Lipschitz  for all $x\in \mathcal{D}$ and the convex set $\mathcal{C}$ satisfying $\|\mathcal{C}\|_2\leq1$,  then with probability at least $1-2\beta$,  we have:
$\text{Err}_{\mathcal{P}} (\theta_{\text{priv}})\leq\tilde{O}\Big (\frac{(\sqrt{\log 1/\beta})^{\frac{1}{4}}D_p^{\frac{1}{4}}p^{\frac{1}{8}}c_1^{p^2}}{\beta n^{\frac{1}{12}}\epsilon^{\frac{1}{4}}}\Big )$.
That is, if we have sample complexity $n=\tilde{\Omega}\big(\max\{\frac{\log \frac{1}{\beta}c^{p^2}}{\epsilon^2 D_{p}^2},(\sqrt{\log 1/\beta})^{3}D_p^{3}p^{\frac{3}{2}}c_2^{p^2}\epsilon^{-3}\alpha^{-12}\beta^{-12}\big)$, then we have $\text{Err}_{\mathcal{P}}(\theta_{\text{priv}})\leq \alpha$. Here $c, c_1, c_2$ are some constants.
\end{theorem}
\begin{lemma}\label{lemma:2}\cite{shalev2009stochastic}
	If the loss function $\ell$ is L-Lipschitz and $\mu$-strongly convex, then with probability at least $1-\beta$ over the randomness of sampling the data set $\mathcal{D}$, the following is true,
	\begin{equation}
	\text{Err}_{\mathcal{P}}(\theta)\leq \sqrt{\frac{2L^2}{\mu}}\sqrt{\text{Err}_\mathcal{D}(\theta)}+\frac{4L^2}{\beta \mu n}.
	\end{equation}
\end{lemma}

\begin{proof}[Proof of Theorem 10]
	For the general convex loss function $\ell$, we let $\hat{\ell}(\theta;x)=\ell(\theta;x)+\frac{\mu}{2}\|\theta\|^2$ for some $\mu>0$. Note that in this case the new empirical risk becomes $\bar{L}(\theta;D)=\hat{L}(\theta;D)+\frac{\mu}{2}\|\theta\|^2$.  Since $\frac{\mu}{2}\|\theta\|^2$ does not depend on the dataset, we can still use the Bernstein polynomial approximation for the original empirical risk $\hat{L}(\theta;D)$ as in Algorithm 2, and the error bound for $\bar{L}(\theta;D)$ is the same.  Thus, we can get the population excess risk of the loss function $\hat{\ell}$, $\text{Err}_{\mathcal{P},\hat{\ell}}(\theta_{\text{priv}})$ by Corollary 1 and we have 
	the following relation,
	$$\text{Err}_{\mathcal{P},\ell}(\theta_{\text{priv}})\leq \text{Err}_{\mathcal{P},\hat{\ell}}(\theta_{\text{priv}})+\frac{\mu}{2}.$$ By the above lemma for $\text{Err}_{\mathcal{P},\hat{\ell}}(\theta_{\text{priv}})$, where $\hat{\ell}(\theta;x)$ is $1+\|\mathcal{C}\|_2=O(1)$-Lipschitz, thus we have the following,
	\begin{equation*}
	\text{Err}_{\mathcal{P},\ell}(\theta_{\text{priv}})\leq \tilde{O}(\sqrt{\frac{2}{\mu}}{\frac{\log^{\frac{1}{8}} \frac{1}{\beta}D_p^{\frac{1}{4}}p^{\frac{1}{8}}c^{(p+1)p}}{n^{\frac{1}{8}}\epsilon^{\frac{1}{4}}}}+\frac{4}{\beta \mu n}+\frac{\mu}{2}).
	\end{equation*}
	Taking $\mu=O(\frac{1}{\sqrt[12]{ n}})$, we get 
	\begin{equation*} 
	\text{Err}_{\mathcal{P},\ell}(\theta_{\text{priv}})\leq\tilde{O}(\frac{\log^{\frac{1}{8}} \frac{1}{\beta} D_p^{\frac{1}{4}}p^{\frac{1}{8}}c^{p^2}}{\beta n^{\frac{1}{12}}\epsilon^{\frac{1}{4}}}).
	\end{equation*}
	Thus, we have the theorem. 
\end{proof}

\section{Details in Section 5}
\begin{algorithm*}[h]
	\caption{Player-Efficient Local Bernstein Mechanism with $O(\log n)$-bits communication per player}
	\label{Dalg:3}
	\begin{algorithmic}[1]
		\State{\bfseries Input:} Each user $i\in [n]$ has data $x_i\in \mathcal{D}$, privacy parameter $\epsilon$, public loss function $\ell:[0,1]^p \times \mathcal{D}\mapsto [0,1]$, and parameter $k$( we will specify it later).
		\State {\bfseries Preprocessing:} 
		\State Construct the grid $\mathcal{T}=\{\frac{v_1}{k},\frac{v_2}{k},\cdots,\frac{v_p}{k}\}_{v_1,v_2,\cdots,v_p}$, where $\{v_1,v_2,\cdots,v_p\}=\{0,1,\cdots,k\}^p$.
		\State Discretize the interval $[0,1]$ with grid steps $O(\frac{1}{n\epsilon}\sqrt{\frac{d}{n}\log(\frac{d}{\beta})})$. Denote the set of grids by $\mathcal{G}$.
		\State Randomly partition  $[n]$ in to $d=(k+1)^p$ subsets $I_1,I_2,\cdots,I_d$, with each subset $I_j$ corresponding to a grid in $\mathcal{T}$ denoted  as 
		$\mathcal{T}(j)$.	
		\For{Each Player $i\in[n]$}
		\State
		Find the subset $I_{\ell}$ such that $i\in  I_{\ell}$. Calculate $v_i=\ell(\mathcal{T}(l);x_i)$.
		\State
		Denote $z_i=v_i+\text{Lap}(\frac{1}{\epsilon})$, round $z_i$ into the grid set $\mathcal{G}$, and let the resulting one be $\tilde{z_i}$.
		\State
		Send $(\tilde{z_i},\ell)$.	
		\EndFor
		\For{The Server}
		\For {Each $\ell\in[d]$}
		\State Compute $v_{\ell}=\frac{n}{|I_{\ell}|}\sum_{i\in I_{\ell}}\tilde{z_i}$.
		\State Denote the corresponding grid point $(\frac{v_1}{k},\frac{v_2}{k},\cdots,\frac{v_p}{k})\in \mathcal{T}$ as $\ell$; then let $\hat{L}((\frac{v_1}{k},\frac{v_2}{k},\cdots,\frac{v_p}{k});D)=v_{\ell}$.
		\EndFor
		\State Construct perturbed Bernstein polynomial of the empirical loss $\tilde{L}$ as in Algorithm 2, where each $\hat{L}((\frac{v_1}{k},\frac{v_2}{k},\cdots,\frac{v_p}{k});D)$ is replaced by $\tilde{L}((\frac{v_1}{k},\frac{v_2}{k},\cdots,\frac{v_p}{k});D)$. Denote the function as $\tilde{L}(\cdot,D)$. 
		\State Compute $\theta_{\text{priv}}=\arg\min_{\theta\in \mathcal{C}}\tilde{L}(\theta;D)$.
		\EndFor
	\end{algorithmic}
\end{algorithm*}

\begin{proof}[Proof of Theorem 4]
	By \cite{bassily2015local} it is $\epsilon$-LDP. The time complexity and communication complexity is obvious. As in \cite{bassily2015local}, it is sufficient to  show that the LDP-AVG is sampling resilient. Here the STAT is the average, and $\phi(x,y)$ is $\max_{j\in[p]}|[x]_j-[y]_j|$. By Lemma 2, we can see that with probability  at least $1-\beta$, $\phi(\text{Avg}(v_1,v_2,\cdots,v_n); a)=O(\frac{bp}{\sqrt{n}\epsilon}\sqrt{\log \frac{p}{\beta}})$. Now let $\mathcal{S}$ be the set obtained by sampling each point $v_i, i\in[n]$ independently with probability $\frac{1}{2}$. Note that by Lemma 2, we have on the subset $\mathcal{S}$. If $|S|\geq \Omega(\max\{p\log(\frac{p}{\beta}), \frac{1}{\epsilon^2}\log \frac{1}{\beta}\})$with probability $1-\beta$, $\phi(\text{Avg}(\mathcal{S}); \text{LDP-AVG}(\mathcal{S}))=O(\frac{b\sqrt{p}}{\sqrt{|\mathcal{S}|}\epsilon}\sqrt{\log \frac{p}{\beta}})$. Now by Hoeffdings Inequality, we can get $|n/2-|\mathcal{S}||\leq \sqrt{n\log \frac{4}{\beta}}$ with probability $1-\beta$. Also since $n=\Omega (\log\frac{1}{\beta})$, we know that $|\mathcal{S}|\geq O(n)\geq \Omega(p\log(\frac{p}{\beta}))$ is true. Thus, with probability at least $1-2\beta$, $\phi(\text{Avg}(\mathcal{S}); \text{LDP-AVG}(\mathcal{S}))=O(\frac{bp}{\sqrt{n}\epsilon}\sqrt{\log \frac{p}{\beta}})$.\par 
	Actually, we can also get $\phi(\text{Avg}(\mathcal{S});\text{Avg}(v_1,v_2,\cdots,v_n))\leq O(\frac{bd}{\sqrt{n}\epsilon}\sqrt{\log \frac{d}{\beta}})$. We now first assume that $v_i\in \mathbb{R}$. Note that 
	$\text{Avg}(\mathcal{S})=\frac{v_1x_1+\cdots+v_nx_n}{x_1+\cdots+x_n}$, where each $x_i\sim \text{Bernoulli}(\frac{1}{2})$. Denote $M=x_1+x_2+\cdots+x_n$, by Hoeffdings Inequality, we have with probability at least $1-\frac{\beta}{2}$, $|M-\frac{n}{2}|\leq \sqrt{n\log \frac{4}{\beta}}$. Denote $N=v_1x_1+\cdots+v_nx_n$. Also, by Hoeffdings inequality, with probability at least $1-\beta$, we get $|N-\frac{v_1+\cdots+v_n}{2}|\leq b\sqrt{n\log\frac{2}{\beta}}$. Thus,   with probability at least $1-\beta$, we have:
	\begin{equation}
	|\frac{N}{M}-\frac{v_1+\cdots+v_n}{n}|\leq \frac{|N-\sum_{i=1}^{n}v_i/2|}{M}+|\sum_{i=1}^{n}v_i/2||\frac{1}{M}-\frac{2}{n}|\leq \frac{|N-\sum_{i=1}^{n}v_i/2|}{M}+\frac{nb}{2}|\frac{1}{M}-\frac{2}{n}|.
	\end{equation}
	The second term $|\frac{1}{M}-\frac{2}{n}|=\frac{|n/2-M|}{M\frac{n}{2}}$. We know from the above $|n/2-M|\leq \sqrt{n\log \frac{4}{\beta}}$. Also  since $n=\Omega (\log\frac{1}{\beta})$, we get $M\geq O(n)$. Thus, $|\frac{1}{M}-\frac{2}{n}|\leq O(\frac{\sqrt{\log \frac{1}{\beta}}}{\sqrt{n}n})$. The upper bound of the second term is $O(\frac{b\sqrt{\log\frac{1}{\beta}}}{\sqrt{n}})$. The same for the first term. For $p$ dimensions, we just choose $\beta=\frac{\beta}{p}$ and take the union. Thus, we have $\phi(\text{Avg}(\mathcal{S});\text{Avg}(v_1,v_2,\cdots,v_n))\leq O(\frac{b}{\sqrt{n}\epsilon}\sqrt{\log \frac{p}{\beta}}) \leq  O(\frac{bp}{\sqrt{n}\epsilon}\sqrt{\log \frac{p}{\beta}})$.\par 
	In summary, we have shown that $\phi(\text{AVG-LDP}(\mathcal{S});\text{Avg}(v_1,v_2,\cdots,v_n))\leq O(\frac{bp}{\sqrt{n}\epsilon}\sqrt{\log \frac{p}{\beta}})$ with probability at least $1-4\beta$.
\end{proof}

Recently, \cite{DBLP:journals/corr/abs-1711-04740} proposed a generic transformation, GenProt, which could transform any $(\epsilon,\delta)$ (so as for $\epsilon$) non-interactive LDP protocol to an $O(\epsilon)$-LDP protocol with the communication complexity for each player being $O(\log \log n)$, which removes the condition of 'sample resilient'  in \cite{bassily2015local}. The detail is in Algorithm 2. The transformation uses $O(n\log \frac{n}{\beta})$ independent public string. The reader is referred to \cite{DBLP:journals/corr/abs-1711-04740} for details. Actually, by Algorithm 2, we can easily get an $O(\epsilon)$-LDP algorithm with the same error bound. 

\begin{theorem}
	With $\epsilon\leq \frac{1}{4}$, under the condition of Corollary 1, Algorithm \ref{Dalg:4} is $10\epsilon$-LDP. If $T=O(\log\frac{n}{\beta})$, then with probability at least $1-2\beta$, Corollary 1 holds. Moreover, the communication complexity of each layer is $O(\log\log n)$ bits, and the computational complexity for each player is $O(\log \frac{n}{\beta})$. 
\end{theorem}

\begin{algorithm}[h]
	\caption{Player-Efficient Local Bernstein Mechanism with $O(\log \log n)$ bits communication complexity.}
	\label{Dalg:4}
	\begin{algorithmic}[1]
		\State {\bfseries Input:} Each user $i\in [n]$ has  data $x_i\in \mathcal{D}$, privacy parameter $\epsilon$, public loss function $\ell:[0,1]^p \times \mathcal{D}\mapsto [0,1]$, and parameter $k, T$. 
		\State {\bfseries Preprocessing:} 
		\State For every $(i, T)\in [n]\times [T]$, generate independent public string ${y_{i,t}}=\text{Lap}(\perp)$.
		\State Construct the grid $\mathcal{T}=\{\frac{v_1}{k},\frac{v_2}{k},\cdots,\frac{v_p}{k}\}_{v_1,v_2,\cdots,v_p}$, where $\{v_1,v_2,\cdots,v_p\}=\{0,1,\cdots,k\}^p$.
		\State Randomly partition $[n]$ in to $d=(k+1)^p$ subsets $I_1,I_2,\cdots,I_d$, with each subset $I_j$ corresponding to an grid in $\mathcal{T}$ denoted  as 
		$\mathcal{T}(j)$.	
		\For{Each Player $i\in[n]$}
		\State
		Find the subset $I_{\ell}$ such that $i\in  I_{\ell}$. Calculate $v_i=\ell(\mathcal{T}(l);x_i)$.
		\State
		For each $t\in [T]$, compute $p_{i,t}=\frac{1}{2}\frac{\text{Pr}[v_i+Lap(\frac{1}{\epsilon})=y_{i,t}]}{\text{Pr}[\text{Lap}(\perp)=y_{i,t}]}$
		\State
		For every $t\in [T]$, if $p_{i,t}\not\in [\frac{e^{-2\epsilon}}{2},\frac{e^{2\epsilon}}{2}]$, then set $p_{i,t}=\frac{1}{2}$.
		\State For every $t\in [T]$, sample a bit $b_{i,t}$ from $\text{Bernoulli}(p_{i,t})$.
		\State Denote $H_i=\{t\in[T]:b_{i,t}=1\}$
		\State If $H_i=\emptyset$, set $H_i=[T]$
		\State Sample $g_i\in H_{i}$ uniformly, and send $g_i$ to the server.	
		\EndFor
		\For{The Server}
		\For {Each $l\in[d]$}
		\State Compute $v_{\ell}=\frac{n}{|I_{\ell}|}\sum_{i\in I_{\ell}}{g_i}$.
		\State Denote the corresponding grid point $(\frac{v_1}{k},\frac{v_2}{k},\cdots,\frac{v_p}{k})\in \mathcal{T}$ as $\ell$; then let $\hat{L}((\frac{v_1}{k},\frac{v_2}{k},\cdots,\frac{v_p}{k});D)=v_{\ell}$.
		\EndFor
		\State Construct perturbed Bernstein polynomial of the empirical loss $\tilde{L}$ as in Algorithm 2. Denote the function as $\tilde{L}(\cdot,D)$. 
		\State Compute $\theta_{\text{priv}}=\arg\min_{\theta\in \mathcal{C}}\tilde{L}(\theta;D)$.
		\EndFor
	\end{algorithmic}
\end{algorithm}
\begin{proof}[Proof of Theorem 5]
	Let $\theta^*=\arg\min_{\theta\in \mathcal{C}}\hat{L}(\theta;D)$, $\theta_{\text{priv}}=\arg\min_{\theta\in \mathcal{C}}\tilde{L}(\theta;D)$. 
	Under the assumptions of $\alpha,n,k,\epsilon,\beta$, we know from the proof of Theorem 3 and Corollary 1 that $\sup_{\theta\in \mathcal{C}}|\tilde{L}(\theta;D)-\hat{L}(\theta;D)|\leq \alpha$. Also  by setting $\epsilon=16348p\alpha$ and $\alpha\leq \frac{1}{16348}\frac{\mu}{p\sqrt{p}}$, we can see that the condition in Lemma 3 holds for $\Delta=\alpha$. So there is an algorithm returns  
	\begin{equation*}
	\tilde{L}(\tilde{\theta}_{\text{priv}};D)\leq \min_{\theta\in \mathcal{C}}\tilde{L}(\theta;D)+O(p\alpha).
	\end{equation*}
	Thus, we have
	\begin{align*}
	\hat{L}(\tilde{\theta}_{\text{priv}};D)-\hat{L}(\theta^*;D)\leq \hat{L}(\tilde{\theta}_{\text{priv}};D)-\tilde{L}(\theta_{\text{priv}};D)+\tilde{L}(\theta_{\text{priv}};D)-\hat{L}(\theta^*;D),
	\end{align*}
	where 
	\begin{align*}
	\hat{L}(\tilde{\theta}_{\text{priv}};D)-\tilde{L}(\theta_{\text{priv}};D)\leq 	\hat{L}(\tilde{\theta}_{\text{priv}};D)-\tilde{L}(\tilde{\theta}_{\text{priv}};D)+\tilde{L}(\tilde{\theta}_{\text{priv}};D)-\tilde{L}(\theta_{\text{priv}};D)\leq \alpha+O(p\alpha)=O(p\alpha).
	\end{align*}
	Also $\tilde{L}(\theta_{\text{priv}};D)-\hat{L}(\theta^*;D)\leq \tilde{L}(\theta^*;D)-\hat{L}(\theta^*;D)\leq \alpha$.  The theorem follows. The running time is determined by $n$. This  is because when we use the algorithm in Lemma 3, we have to use the first order optimization. That is, we have to evaluate some points at $\tilde{L}(\theta;D)$, which will cost at most $O(\text{poly}(n))$ time (note that $\tilde{L}$ is a polynomial with $(k+1)^p\leq n$ coefficients).
\end{proof}
\section{Details pf Section 6}

\begin{proof}[Proof of Theorem 7]
	It is sufficient to prove that
	\begin{equation*}\label{eq:14}
	\sup_{y\in Y_k}|\tilde{p}_D(y)-q_y(D)|\leq \gamma+\frac{T\binom{p+t_k}{t_k}^2\sqrt{\log\frac{\binom{p+t_k}{t_k}}{\beta}}}{\sqrt{n}\epsilon},
	\end{equation*}
	where $T=p^{O(\sqrt{k}\log(\frac{1}{\gamma}))}$.
	Now we denote $p_D\in \mathbb{R}^{\binom{p+t_k}{t_k}}$ as the average of $\tilde{q}_i$. That is, it is the unperturbed version of $\tilde{p}_D$.
	By Lemma \ref{Alemma:1}, we have $\sup_{y\in Y_k}|{p}_D(y)-q_y(D)|\leq \gamma$. Thus it is sufficient to prove that 
	\begin{equation*}
	\sup_{y\in Y_k}|\tilde{p}_D(y)-p_D(y)|\leq \frac{T\binom{p+t_k}{t_k}^2\sqrt{\log\frac{\binom{p+t_k}{t_k}}{\beta}}}{\sqrt{n}\epsilon}.
	\end{equation*}
	Since  $\tilde{p}_D, p_D$ can be  viewed as a vector, we have 
	\begin{equation*}
	\sup_{y\in Y_k}|\tilde{p}_D(y)-p_D(y)|\leq \|\tilde{p}_D-p_D\|_1.
	\end{equation*}
	Also, since each coordinate of $p_D(y)$ is bounded by $T$ by Lemma \ref{Alemma:1}, we can see that if $n=\Omega(\max \{\frac{1}{\epsilon^2}\log \frac{1}{\beta}, \binom{p+t_k}{t_k}\log \binom{p+t_k}{t_k}\log 1/\beta\})$, then with probability at least $1-\beta$, the following is true
	$\|\tilde{p}_D-p_D\|_1\leq \frac{T\binom{p+t_k}{t_k}^2\sqrt{\log\frac{\binom{p+t_k}{t_k}}{\beta}}}{\sqrt{n}\epsilon}$, thus take $\gamma=\frac{\alpha}{2}$ and $\binom{p+t_k}{t_k}=p^{O(t_k)}$. This gives us the theorem.
\end{proof}
\begin{proof}[Proof of Theorem 8]
	Let $t=(\frac{1}{\gamma})^{\frac{1}{h}}$. It is sufficient to prove that $\sup_{q_f\in \mathcal{Q}_{C^h_T}}|\tilde{p}_D \cdot c_f-q_f(D)|\leq \alpha$. Let $p_D$ denote the average of $\{p_i\}_{i=1}^{n}$, {\em i.e.} the unperturbed version of $\tilde{p}_D$. Then by Lemma 5, we have $\sup_{q_f\in \mathcal{Q}_{C^h_T}}|{p}_D \cdot c_f-q_f(D)|\leq \gamma$. Also since $\|c_f\|_{\infty}\leq M$, we have $\sup_{q_f\in \mathcal{Q}_{C^h_T}}|\tilde{p}_D\cdot c_f-p_D\cdot c_f|\leq O(\|\tilde{p}_D-p_D\|_1)$. By Lemma \ref{Alemma:2}, we know that if $n=\Omega(\max\{\frac{1}{\epsilon^2}\log \frac{1}{\beta},t^{2p}\log\frac{1}{\beta}\})$, then $\|\tilde{p}_D-p_D\|_1\leq O(\frac{t^{\frac{5p}{2}}\sqrt{\log(\frac{1}{\beta})}}{\sqrt{n}\epsilon})$ with probability at least $1-\beta$. Thus, we have $\sup_{q_f\in \mathcal{Q}_{C^h_T}}|\tilde{p}_D \cdot c_f-q_f(D)|\leq O(\gamma+ \frac{(\frac{1}{\gamma})^{\frac{5p}{2h}}\sqrt{\log(\frac{1}{\beta})}}{\sqrt{n}\epsilon})$. Taking $\gamma=O((1/\sqrt{n}\epsilon)^{\frac{2h}{5p+2h}})$, we get $\sup_{q_f\in \mathcal{Q}_{C^h_T}}|\tilde{p}_D \cdot c_f-q_f(D)|\leq O(\sqrt{\log(\frac{1}{\beta})}(\frac{1}{\sqrt{n}\epsilon})^{\frac{2h}{5p+2h}})\leq \alpha$.
	The computational cost for answering a query follows from Lemma 2 and $b\cdot c=O(t^p)$.
\end{proof}
\section{Details of Section 7}
\subsection{Modified $\epsilon$-LDP Algorithm}

 \subsection{Proof of Theorem 9}
Before the proof, let us review some definitions. We refer to readers \cite{vershynin2010introduction}\cite{vershynin2015estimation} 
\begin{definition}(Sub-gaussian random vector)
A random variable $a\in \mathbb{R}$ is called subgaussian if there exits a constant $C>0$ such that $\text{Pr}[|a|>t]\leq 2\exp(\frac{-t^2}{C^2})$ for any $t>0$. Also we say a random vector $a\in \mathbb{R}^p$ is subgaussian if the one dimensional marginals $\langle a,b \rangle$ are subgaussian random variable for all $b\in \mathbb{R}^p$.
\end{definition}
For any subgaussian random variable (vector) we have subgaussian norm.
\begin{definition}
The $\psi_2$ norm of a subgaussian random variable $a\in \mathbb{R}$, denoted by $\|a\|_{\psi_2}$ is:
\begin{equation*}
\|a\|_{\psi_2}=\inf\{t>0:\mathbb{E}[\exp(\frac{|a|^2}{t^2})]\leq 2\}.
\end{equation*}
The $\psi_2$ norm of a subgaussian vector $a\in\mathbb{R}^p$ is:
\begin{equation*}
\|a\|_{\psi_2}=\sup_{b\in\mathcal{S}^{p-1}}\|\langle a, b\rangle\|_{\psi_2}.
\end{equation*}
Note that when $a$ is normal random Gaussian vector, then $\|a\|_{\psi_2}$ is bounded by a constant \cite{vershynin2010introduction}.
\end{definition}
\begin{definition}[Gaussian Width]
Given a closed set $S\subset \mathbb{R}^d$, its Gaussian width is defined as:
\begin{equation*}
\mathcal{G}_{\mathcal{C}}=\mathbb{E}_{g\sim\mathcal{N}(0,1)^d}[\sup_{a\in S}\langle a, g\rangle].
\end{equation*}

\end{definition}
	The Minkowski norm (denoted by $||\cdot||_{\mathcal{C}}$) with respect to a centrally symmetric convex set $\mathcal{C}\subseteq \mathbb{R}^p$ is defined as follows. For any vector $v\in \mathbb{R}^p$,
	\[||\cdot||_{\mathcal{C}}=\min\{r \in  \mathbb{R}^+:v\in r\mathcal{C}\}.\]
	The main theorem of dimension reduction is as the following:
\begin{theorem}\label{them1:3}
Let $\tilde{\Phi}\in \mathbb{R}^{m\times p}$ be an random matrix, whose rows are i.i.d mean-zero, isotropic, subgaussian random variable in $\mathbb{R}^d$ with $\psi=\|\Phi_i\|_{\psi_2}$. Let $\Phi=\frac{1}{\sqrt{m}}\tilde{\Phi}$. let $S$ be a set if points in $R^d$. Then there is a constant $C>0$ such that for any $0<\gamma,\beta<1$.
\begin{equation*}
\text{Pr}[\sup_{a\in S}|\|\Phi a\|^2-\|a\|^2\leq \gamma\|a\|^2]\leq \beta,
\end{equation*}
provided that $m\geq \frac{C\psi^4}{\gamma^2}\max\{\mathcal{G}_{\mathcal{S}},\log(1/\beta)\}^2$.
\end{theorem}

The proof follows \cite{kasiviswanathan2016efficient}, here we rephrase it for completeness.
\begin{lemma}
Let $\Phi$ be a random matrix as defined in Theorem \ref{them1:3} with $m=\Theta((\frac{\psi^4}{\gamma^2}\log(n/\beta))$ for $\beta>0$. Then with probability at least $1-\beta$, $f(\langle \Phi y_i, \cdot\rangle, z_i)$ is 2-Lipschitz over the domain $\Phi\mathcal{C}$ for each $i\in[n]$.
\end{lemma}
Now since $\mathcal{C}$ is convex, so $\Phi\mathcal{C}$ is also convex, furthermore, by Theorem \ref{them1:3} we have if $m=\Theta(\frac{\psi^4}{\gamma}\max\{\mathcal{G}_{\mathcal{C}}, \log\frac{1}{\beta}\}$ for $\gamma<1$, then $\|\Phi\mathcal{C}\|_2\leq O(1)$. Thus after compression, the loss function and constrained set still satisfy the assumption in 
\cite{smith2017interaction}. So by \cite{smith2017interaction}  we have:
\begin{theorem}
With probability at least $1-\beta$,
\begin{equation}\label{aeq1}
\frac{1}{n}\sum_{i=1}^{n}f(\langle \Phi y_i, \bar{w}\rangle, z_i)-\min_{w\in\mathcal{C}}\frac{1}{n}\sum_{i=1}^{n}f(\langle \Phi y_i, \Phi w\rangle, z_i)\leq \tilde{O}((\frac{\log^2(1/\beta)\sqrt{m}}{n\epsilon^2})^{\frac{1}{m+1}}).
\end{equation}

\end{theorem}
We now have the following by using Lipschitz and Theorem \ref{them1:3}:
\begin{lemma}
Let $\Phi$ be the random matrix in Theorem \ref{them1:3} with $m=\Theta((\frac{\psi^4}{\gamma^2}\log(n/\beta))$ for $\beta>0$. Then for any $\hat{w}\in\mathcal{C}$, with probability at least $1-\beta$, we have 
\begin{equation}\label{aeq2}
\min_{w\in\mathcal{C}}\frac{1}{n}\sum_{i=1}^{n}f(\langle \Phi y_i, \Phi w\rangle, z_i)\leq \frac{1}{n}\sum_{i=1}^{n}f(\langle  y_i, \hat{w}\rangle, z_i)+\gamma\|\mathcal{C}\|_2.
\end{equation}

\end{lemma}
Also we denote $\theta\in\mathcal{C}$ such that $\Phi\theta=\bar{w}$. We have the following lemma by Theorem \ref{them1:3}.
\begin{lemma}
Let $\Phi$ be a random matrix as in Theorem \ref{them1:3} with $m=\Theta((\frac{\psi^4}{\gamma^2}(\mathcal{G}_{\mathcal{C}}+\sqrt{\log{n}})^2\log(n/\beta)$ for $\beta>0$, then with probability at least $1-\beta$:
\begin{equation}\label{aeq3}
|\frac{1}{n}\sum_{i=1}^{n}f(\langle \Phi y_i, \Phi\theta \rangle, z_i)- \frac{1}{n}\sum_{i=1}^{n}f(\langle  y_i, \theta \rangle, z_i)|\leq \gamma\|\mathcal{C}\|_2.
\end{equation}

\end{lemma}
We will establish the connection of $\theta$ and $w^{\text{priv}}$ by the following lemma:
\begin{theorem}\cite{vershynin2010introduction}
Let $\Phi$ be a random matrix in Theorem \ref{them1:3}. Let $\mathcal{C}$ be a convex set. Given $v=\Phi u$, and let $\hat{u}$ be the solution to the following convex program:$\min_{u'\in \mathcal{R}^p}\|u'\|_{\mathcal{C}}$ subject to $\Phi u'=v$. Then for any $\beta>0$, with probability at least $1-\beta$,
\begin{equation}\label{aeq4}
\sup_{u:v=\Phi u}\|u-\hat{u}\|_2\leq O(\frac{\psi^4\mathcal{G}_{\mathcal{C}}}{\sqrt{m}}+\frac{\psi^4\|\mathcal{C}\|_2\sqrt{\log(1/\beta)}}{\sqrt{m}}).
\end{equation}

\end{theorem}
Combing (\ref{aeq1})(\ref{aeq2})(\ref{aeq3})(\ref{aeq4}). We get the following bound:
\begin{theorem}
Under the assumption above. Set $m=\Theta(\frac{\psi^4(\mathcal{G}_{\mathcal{C}}+\sqrt{\log n})^2\log(n/\beta)}{\gamma^2})$ for $\gamma<1$. Then with probability at least $1-\beta$, 
\begin{equation}
\text{Err}_D(w^{\text{priv}})=\tilde{O}\big(\big(\frac{\log^2(1/\beta)\sqrt{m}}{n\epsilon^2}\big)^{\frac{1}{1+m}}+\gamma),
\end{equation}
where $\psi$ is the subgaussian norm of the distribution of $\Phi$, $\mathcal{G}_{\mathcal{C}}$ is the Gaussian width of $\mathcal{C}$.
\end{theorem}
Then take $\gamma$ as in the Theorem, we can get the proof.
For the corollary we will use the property that $\mathcal{G}_{\mathcal{C}}=O(\sqrt{\log p})$.

\end{document}